\documentclass{article}

\makeatletter
\newif\if@neuripsfinal\@neuripsfinalfalse
\newif\if@natbib\@natbibtrue
\newif\if@preprint\@preprintfalse
\relax
\@neuripsfinaltrue  

\newif\if@submission\@submissiontrue
\if@neuripsfinal\@submissionfalse\fi
\if@preprint\@submissionfalse\fi

\newcommand{\@neuripsordinal}{40th}
\newcommand{\@neuripsyear}{2026}
\newcommand{\@neuripslocation}{Sydney}

\usepackage{environ}
\newcommand{\acksection}{\section*{Acknowledgments and Disclosure of Funding}}
\NewEnviron{ack}{%
  \acksection
  \BODY
}

\if@natbib
  \RequirePackage{natbib}
\fi

\usepackage[verbose=true,letterpaper]{geometry}
\AtBeginDocument{
  \newgeometry{
    textheight=9in,
    textwidth=5.5in,
    top=1in,
    headheight=12pt,
    headsep=25pt,
    footskip=30pt
  }
  \@ifpackageloaded{fullpage}
    {\PackageWarning{neurips_2024}{fullpage package not allowed! Overwriting formatting.}}
    {}
}

\renewcommand{\normalsize}{%
  \@setfontsize\normalsize\@xpt\@xipt
  \abovedisplayskip      7\p@ \@plus 2\p@ \@minus 5\p@
  \abovedisplayshortskip \z@ \@plus 3\p@
  \belowdisplayskip      \abovedisplayskip
  \belowdisplayshortskip 4\p@ \@plus 3\p@ \@minus 3\p@
}
\normalsize
\renewcommand{\small}{%
  \@setfontsize\small\@ixpt\@xpt
  \abovedisplayskip      6\p@ \@plus 1.5\p@ \@minus 4\p@
  \abovedisplayshortskip \z@  \@plus 2\p@
  \belowdisplayskip      \abovedisplayskip
  \belowdisplayshortskip 3\p@ \@plus 2\p@   \@minus 2\p@
}
\renewcommand{\footnotesize}{\@setfontsize\footnotesize\@ixpt\@xpt}
\renewcommand{\scriptsize}{\@setfontsize\scriptsize\@viipt\@viiipt}
\renewcommand{\tiny}{\@setfontsize\tiny\@vipt\@viipt}
\renewcommand{\large}{\@setfontsize\large\@xiipt{14}}
\renewcommand{\Large}{\@setfontsize\Large\@xivpt{16}}
\renewcommand{\LARGE}{\@setfontsize\LARGE\@xviipt{20}}
\renewcommand{\huge}{\@setfontsize\huge\@xxpt{23}}
\renewcommand{\Huge}{\@setfontsize\Huge\@xxvpt{28}}

\providecommand{\section}{}
\renewcommand{\section}{%
  \@startsection{section}{1}{\z@}%
                {-2.0ex \@plus -0.5ex \@minus -0.2ex}%
                { 1.5ex \@plus  0.3ex \@minus  0.2ex}%
                {\large\bf\raggedright}%
}
\providecommand{\subsection}{}
\renewcommand{\subsection}{%
  \@startsection{subsection}{2}{\z@}%
                {-1.8ex \@plus -0.5ex \@minus -0.2ex}%
                { 0.8ex \@plus  0.2ex}%
                {\normalsize\bf\raggedright}%
}
\providecommand{\subsubsection}{}
\renewcommand{\subsubsection}{%
  \@startsection{subsubsection}{3}{\z@}%
                {-1.5ex \@plus -0.5ex \@minus -0.2ex}%
                { 0.5ex \@plus  0.2ex}%
                {\normalsize\bf\raggedright}%
}
\providecommand{\paragraph}{}
\renewcommand{\paragraph}{%
  \@startsection{paragraph}{4}{\z@}%
                {1.5ex \@plus 0.5ex \@minus 0.2ex}%
                {-1em}%
                {\normalsize\bf}%
}
\providecommand{\subparagraph}{}
\renewcommand{\subparagraph}{%
  \@startsection{subparagraph}{5}{\z@}%
                {1.5ex \@plus 0.5ex \@minus 0.2ex}%
                {-1em}%
                {\normalsize\bf}%
}

\newlength{\@neuripsabovecaptionskip}
\newlength{\@neuripsbelowcaptionskip}

\renewenvironment{table}
  {\setlength{\abovecaptionskip}{\@neuripsbelowcaptionskip}%
   \setlength{\belowcaptionskip}{\@neuripsabovecaptionskip}%
   \@float{table}}
  {\end@float}

\renewcommand{\footnoterule}{\kern-3\p@ \hrule width 12pc \kern 2.6\p@}
\def\@listi  {\leftmargin\leftmargini}
\def\@listii {\leftmargin\leftmarginii
              \labelwidth\leftmarginii
              \advance\labelwidth-\labelsep
              \topsep  2\p@ \@plus 1\p@    \@minus 0.5\p@
              \parsep  1\p@ \@plus 0.5\p@ \@minus 0.5\p@
              \itemsep \parsep}
\def\@listiii{\leftmargin\leftmarginiii
              \labelwidth\leftmarginiii
              \advance\labelwidth-\labelsep
              \topsep    1\p@ \@plus 0.5\p@ \@minus 0.5\p@
              \parsep    \z@
              \partopsep 0.5\p@ \@plus 0\p@ \@minus 0.5\p@
              \itemsep \topsep}
\def\@listiv {\leftmargin\leftmarginiv
              \labelwidth\leftmarginiv
              \advance\labelwidth-\labelsep}
\def\@listv  {\leftmargin\leftmarginv
              \labelwidth\leftmarginv
              \advance\labelwidth-\labelsep}
\def\@listvi {\leftmargin\leftmarginvi
              \labelwidth\leftmarginvi
              \advance\labelwidth-\labelsep}

\providecommand{\maketitle}{}
\renewcommand{\maketitle}{%
  \par
  \begingroup
    \renewcommand{\thefootnote}{\fnsymbol{footnote}}
    \renewcommand{\@makefnmark}{\hbox to \z@{$^{\@thefnmark}$\hss}}
    \long\def\@makefntext##1{%
      \parindent 1em\noindent
      \hbox to 1.8em{\hss $\m@th ^{\@thefnmark}$}##1
    }
    \thispagestyle{empty}
    \@maketitle
    \@thanks
    \@notice
  \endgroup
  \let\maketitle\relax
  \let\thanks\relax
}

\newcommand{\@toptitlebar}{
  \hrule height 4\p@
  \vskip 0.25in
  \vskip -\parskip%
}
\newcommand{\@bottomtitlebar}{
  \vskip 0.29in
  \vskip -\parskip
  \hrule height 1\p@
  \vskip 0.09in%
}

\providecommand{\@maketitle}{}
\renewcommand{\@maketitle}{%
  \vbox{%
    \hsize\textwidth
    \linewidth\hsize
    \vskip 0.1in
    \@toptitlebar
    \centering
    {\LARGE\bf \@title\par}
    \@bottomtitlebar
    \if@submission
      \begin{tabular}[t]{c}\bf\rule{\z@}{24\p@}
        Anonymous Author(s) \\
        Affiliation \\
        Address \\
        \texttt{email} \\
      \end{tabular}%
    \else
      \def\And{%
        \end{tabular}\hfil\linebreak[0]\hfil%
        \begin{tabular}[t]{c}\bf\rule{\z@}{24\p@}\ignorespaces%
      }
      \def\AND{%
        \end{tabular}\hfil\linebreak[4]\hfil%
        \begin{tabular}[t]{c}\bf\rule{\z@}{24\p@}\ignorespaces%
      }
      \begin{tabular}[t]{c}\bf\rule{\z@}{24\p@}\@author\end{tabular}%
    \fi
    \vskip 0.3in \@minus 0.1in
  }
}

\newcommand{\ftype@noticebox}{8}
\newcommand{\@notice}{%
  \enlargethispage{2\baselineskip}%
  \@float{noticebox}[b]%
    \footnotesize\@noticestring%
  \end@float%
}

\renewenvironment{abstract}%
{%
  \vskip 0.075in%
  \centerline%
  {\large\bf Abstract}%
  \vspace{0.5ex}%
  \begin{quote}%
}
{
  \par%
  \end{quote}%
  \vskip 1ex%
}

\newcommand{\answerTODO}[1][]{\textcolor{red}{\bf [TODO]}}
\newcommand{\justificationTODO}[1][]{\textcolor{red}{\bf [TODO]}}

\if@preprint
  \newcommand{\@noticestring}{%
    Preprint. Under review.%
  }
\else
  \if@neuripsfinal
    \newcommand{\@noticestring}{%
      \@neuripsordinal\/ Conference on Neural Information Processing Systems
      (NeurIPS \@neuripsyear).
    }
  \else
    \newcommand{\@noticestring}{%
      Submitted to \@neuripsordinal\/ Conference on Neural Information
      Processing Systems (NeurIPS \@neuripsyear). Do not distribute.%
    }

    \NewEnviron{hide}{}

    \AtBeginDocument{%
      \@ifpackageloaded{amsmath}{%
        \newcommand*\patchAmsMathEnvironmentForLineno[1]{%
          \expandafter\let\csname old#1\expandafter\endcsname\csname #1\endcsname
          \expandafter\let\csname oldend#1\expandafter\endcsname\csname end#1\endcsname
          \renewenvironment{#1}%
                          {\linenomath\csname old#1\endcsname}%
                          {\csname oldend#1\endcsname\endlinenomath}%
        }%
        \newcommand*\patchBothAmsMathEnvironmentsForLineno[1]{%
          \patchAmsMathEnvironmentForLineno{#1}%
          \patchAmsMathEnvironmentForLineno{#1*}%
        }%
        \patchBothAmsMathEnvironmentsForLineno{equation}%
        \patchBothAmsMathEnvironmentsForLineno{align}%
        \patchBothAmsMathEnvironmentsForLineno{flalign}%
        \patchBothAmsMathEnvironmentsForLineno{alignat}%
        \patchBothAmsMathEnvironmentsForLineno{gather}%
        \patchBothAmsMathEnvironmentsForLineno{multline}%
      }
      {}
    }
  \fi
\fi
\makeatother

\usepackage{amsmath}
\usepackage{amssymb}
\usepackage{amsfonts}
\usepackage{amsthm}
\usepackage{mathtools}

\usepackage{booktabs}
\usepackage{longtable}
\usepackage{multirow}

\usepackage{graphicx}
\usepackage{xcolor}

\usepackage{algorithm}
\usepackage{algpseudocode}

\usepackage{hyperref}
\usepackage{cleveref}

\newcommand{\R}{\mathbb{R}}
\newcommand{\F}{\mathbb{F}}
\newcommand{\E}{\mathbb{E}}
\newcommand{\TopK}{\operatorname{TopK}}
\newcommand{\ReLU}{\operatorname{ReLU}}
\newcommand{\softmax}{\operatorname{softmax}}
\newcommand{\KL}{\operatorname{KL}}

\newtheorem{theorem}{Theorem}
\newtheorem{corollary}[theorem]{Corollary}
\newtheorem{proposition}[theorem]{Proposition}
\newtheorem{definition}[theorem]{Definition}
\newtheorem{lemma}[theorem]{Lemma}
\theoremstyle{remark}
\newtheorem{observation}[theorem]{Observation}
\crefname{observation}{observation}{observations}

\title{VALSE: Vertical Adaptive Layer Skipping for Efficient Inference in Large Language Models}

\author{Jia-Dong Zhang\\
jiadongzh@gmail.com}

\begin{document}

\maketitle

\begin{abstract}
This paper establishes a theoretical framework for vertical adaptive layer
skipping, proving three foundational results: (i)~an \emph{Expected FLOPs}
formula (\cref{thm:exp_flops}) giving a closed-form expression for the
computational cost of arbitrary per-sample skip schedules as a function of
layer-wise skip probabilities; (ii)~\emph{function-space superset}
(\cref{thm:full_subset}) and \emph{strict inclusion} (\cref{thm:strict})
theorems showing that skip-layer models are strictly contained in---yet
meaningfully approximate---the full-layer function space, with an explicit
separating example; and (iii)~a \emph{structural duality} between VALSE and
Mixture-of-Experts architectures (\cref{prop:duality}), positioning vertical
depth-wise sparsity as the orthogonal counterpart to horizontal width-wise
sparsity. Building on this theory, we propose VALSE (Vertical Adaptive Layer Skipping
for Efficiency), a per-sample, non-contiguous layer skipping method: a
lightweight difficulty estimator scores each input from the first few layers,
and per-layer gates selectively skip redundant layers---including arbitrary
middle layers while retaining deeper ones---so that only the necessary depth
is activated for each input, whose feasibility is preliminarily assessed at
prototype scale.
\end{abstract}

\section{Introduction}
\label{sec:intro}

Modern large language models (e.g., GPT-4~\citep{achiam2023gpt}, LLaMA~\citep{touvron2023llama}, DeepSeek-MoE~\citep{deepseekmoe2024}) stack dozens to hundreds of Transformer blocks.
Regardless of input difficulty---from simple sentiment classification to complex multi-step problems---every sample must traverse all layers, incurring identical computation.

In the \textbf{horizontal (width) dimension}, Mixture-of-Experts (MoE) achieves sparse activation by selecting $K$ of $E$ experts per token per layer.
However, in the \textbf{vertical (depth) dimension}, the number of layers remains fixed for all inputs: MoE makes each layer ``cheaper'' but does not reduce the \emph{number} of layers.
Depth thus constitutes an underexploited orthogonal direction for conditional computation.

The core idea is natural: if simple inputs do not require the full representational power of all layers, can we \textbf{vertically skip certain layers} based on input difficulty, activating only the necessary depth?
MoE achieves width-wise sparsity; vertical layer skipping achieves depth-wise sparsity---the two form orthogonal dimensions.

\paragraph{Limitations of existing methods.}
Prior work on adaptive depth falls into several categories, each with structural limitations: \emph{early exit}~\citep{deeBERT2020,fastBERT2020,calm2022} is truncation-based and cannot skip a middle layer while retaining a deeper one; \emph{adaptive computation time}~\citep{act2016,universalTransformer2019} adds shared-weight recurrence rather than skipping fixed layers; \emph{structured dropout}~\citep{layerDrop2020,layerskip2024} selects subnetworks statically or per token, not per-sample by difficulty; and \emph{token-level routing} (MoD~\citep{mod2024}) operates at token granularity and lacks an explicit difficulty signal.
Crucially, no existing method achieves \emph{per-sample, non-contiguous layer skipping} (skipping arbitrary middle layers while retaining deeper ones) in conjunction with a unified theoretical duality to horizontal MoE.
A detailed comparison is provided in \cref{sec:related} and \cref{sec:gaps}.

\paragraph{Contributions.}
We propose VALSE (Vertical Adaptive Layer Skipping for Efficiency), with three contributions:
\begin{enumerate}
\item \textbf{Theoretical framework.} We formalize vertical sparse conditional computation through three fully proven results: the \emph{Expected FLOPs} theorem (\cref{thm:exp_flops}), which gives a closed-form expression for the computational cost of arbitrary per-sample skip schedules; the \emph{function-space superset} (\cref{thm:full_subset}) and \emph{strict containment} (\cref{thm:strict}) theorems, establishing that skip-layer models are strictly contained in the full-layer function space with an explicit separating example; and the \emph{MoE--VALSE structural duality} (\cref{prop:duality}) together with a \emph{unified two-dimensional sparsity} representation (\cref{def:2d_sparse,thm:unified}), positioning VALSE as the depth-wise counterpart to horizontal MoE (\cref{sec:theory}).
\item \textbf{Method design.} A lightweight, interpretable difficulty estimator scores each sample from the first $K_{\text{est}}$ layers, feeding three routing variants (hard, soft, global) that produce per-sample, non-contiguous layer retention sets---extending conditional computation from the width to the depth dimension. The method is supported by four auxiliary losses (layer-utilization balancing, efficiency, output consistency via KL distillation, and Z-loss) and a three-stage curriculum learning schedule (\cref{sec:difficulty,sec:routing,sec:training}).
\item \textbf{Proof-of-concept validation.} A prototype-scale ($12$-layer) evaluation on SST-2 provides initial feasibility evidence for the layer-skipping mechanism. We explicitly frame these results as preliminary and emphasize that the core empirical hypothesis---difficulty-driven differentiated routing---is \emph{not yet validated} at the current scale; full-scale evaluation on larger, pre-trained models is necessary to confirm the approach (\cref{sec:experiments}).
\end{enumerate}

\section{Related Work}
\label{sec:related}

We survey adaptive depth computation and horizontal MoE routing, identifying the research gaps that motivate VALSE.

\subsection{Adaptive Depth Methods}
\label{subsec:adaptive_depth}

Existing work falls into five categories, summarized in \cref{tab:adaptive_depth_summary} and contrasted along structural axes in \cref{tab:comparison}.
\textbf{Early exit} (DeeBERT~\citep{deeBERT2020}, FastBERT~\citep{fastBERT2020}, PABEE~\citep{pabee2020}) attaches confidence classifiers at intermediate layers; all are \emph{truncation-based}---once the model exits, all subsequent layers are skipped.
CALM~\citep{calm2022} extends early exit to generative LMs but remains truncation-based.
\textbf{Recurrent adaptive computation} (ACT~\citep{act2016}, Universal Transformers~\citep{universalTransformer2019}) iterates shared layers to add depth rather than skipping fixed layers; ACT learns a halting distribution over recurrence steps and Universal Transformers couple recurrence with a per-position halting mechanism, but both require shared-weight recurrence and therefore do not apply to standard fixed-layer Transformer stacks.
\textbf{Layer dropout} (LayerDrop~\citep{layerDrop2020}; LayerSkip~\citep{layerskip2024} with self-speculative decoding~\citep{leviathan2023fast}) supports subnetwork extraction but uses static or per-token selection, not per-sample difficulty-driven routing.
Layer distillation~\citep{layerDistill2020} produces static depth variants.
\textbf{Token-level routing:} MoD~\citep{mod2024} decides per-token per-layer whether to process or skip; Drowzee~\citep{drowzee2023} predicts per-sample compute budgets but only controls truncation depth; DynamicBERT~\citep{dynamicbert2021} adjusts width and depth independently without unification.
\textbf{Layer redundancy and structured pruning:} Recent work provides empirical evidence that Transformer layers exhibit significant redundancy. ShortGPT~\citep{shortgpt2024} introduces a Block Influence (BI) score to quantify layer redundancy and shows that up to $50\%$ of layers can be removed with minimal performance loss. Streamline~\citep{streamline2024} and SliceGPT~\citep{slicegpt2024} perform structured layer pruning at the post-training stage, producing static reduced-depth models. MoNE~\citep{mone2024} applies mixture-of-experts routing across encoder--decoder layer pairs, selecting which layer pairs to execute per token. These methods confirm that layer redundancy is pervasive, but all produce \emph{static} or \emph{per-token} depth profiles---none achieves per-sample, difficulty-driven, non-contiguous skipping.

\subsection{Horizontal MoE Routing}
\label{subsec:moe_routing}

MoE achieves width-wise sparse activation: each layer selects $K$ of $N$ experts.
Sparse-gated MoE~\citep{sparselygatedmoe2017} introduced the gating mechanism; GShard~\citep{gshard2020} scaled it to multilingual and multimodal settings; Switch Transformer~\citep{switchtransformer2022} simplified to Top-1 routing and demonstrated that aggressive expert sparsity is compatible with stable pretraining.
Mixtral~\citep{mixtral2024} and DeepSeek-MoE~\citep{deepseekmoe2024} are recent open representatives, with DeepSeek-MoE in particular introducing fine-grained expert segmentation and shared experts to improve specialization.
Across this family, the routing signal is \emph{semantic and per-token}: the gate reads the token representation and matches it to expert specializations, with no notion of input-level difficulty or sample-level compute budget.
The key observation is that \emph{all MoE methods keep depth fixed}---every token still passes through all $L$ layers---and that the difficulty-awareness dimension, central to VALSE, has no analogue in horizontal MoE routing.
\cref{tab:moe_analogy} makes this orthogonality formal: MoE selects `which experts'' within a layer, whereas VALSE selects `which layers'' across depth, and the two dimensions are mathematically dual but operationally independent.

\subsection{Horizontal vs.\ Vertical Comparison}
\label{subsec:comparison}

\Cref{tab:comparison} summarizes the orthogonality: MoE reduces the number of parameters activated per layer (width); VALSE reduces the number of layers activated (depth).
The two can be combined for two-dimensional sparse activation, and the orthogonality is structural rather than merely intuitive: width-wise sparsity selects a subset of \emph{parallel} compute units within a layer (experts $E_i$ applied to the same hidden state), whereas depth-wise sparsity selects a subset of \emph{sequential} compute units across layers (blocks $F_l$ applied in cascade), so the two selection operators act on disjoint axes of the activation tensor and compose without interference.
This composability is non-trivial in two respects.
First, the routing signals differ in scope: MoE gating is per-token and semantic (matching token content to expert specialization), while VALSE gating is per-sample and difficulty-driven (matching input complexity to retained depth), so a joint router must expose both a per-token width gate and a per-sample depth gate without one collapsing into the other.
Second, the load-balancing objectives differ in target: horizontal MoE prevents expert under-utilization (few experts absorbing most tokens), whereas VALSE prevents layer under-utilization (most samples collapsing to the same minimal-depth path), so a unified training objective must regularize both the expert-utilization distribution and the per-layer activation frequency simultaneously.
When these conditions hold, the two dimensions yield a multiplicative compute reduction: under a width sparsity $s_w$ (fraction of experts activated) and a depth sparsity $s_d$ (fraction of layers retained), the activated parameter count scales as $(1-s_w)(1-s_d)$ of the dense network, a property formalized in \cref{thm:unified} as the unified width$\times$depth sparsity framework.
No prior method surveyed above attains both dimensions simultaneously: horizontal MoE keeps depth fixed, and all adaptive-depth methods keep width fixed, so the two-dimensional sparse-activation regime is itself part of the methodological gap that VALSE is designed to occupy.

\begin{table}[htbp]
\centering
\caption{Comparison of adaptive depth methods. Skip granularity and per-sample adaptivity distinguish VALSE from all prior approaches.}
\label{tab:comparison}
\small
\resizebox{\textwidth}{!}{%
\begin{tabular}{lccccc}
\toprule
Method & Sparsity dim. & Skip mode & Granularity & Per-sample? & Difficulty-aware? \\
\midrule
MoE~\citep{mixtral2024,deepseekmoe2024} & Width & --- & Per-token & No & Semantic \\
Early exit~\citep{deeBERT2020,calm2022} & Depth & Truncation & Per-sample & Yes & Confidence \\
LayerDrop~\citep{layerDrop2020} & Depth & Static dropout & Global & No & --- \\
MoD~\citep{mod2024} & Depth & Per-layer & Per-token$\times$layer & No & Implicit \\
ShortGPT~\citep{shortgpt2024} & Depth & Static pruning & Global & No & --- \\
SliceGPT~\citep{slicegpt2024} & Depth & Structured pruning & Global & No & --- \\
\textbf{VALSE (ours)} & \textbf{Depth} & \textbf{Non-contiguous} & \textbf{Per-sample} & \textbf{Yes} & \textbf{Explicit router} \\
\bottomrule
\end{tabular}%
}
\end{table}

\subsection{Research Gaps}
\label{subsec:gaps}

The above analysis reveals four methodological gaps---per-sample non-contiguous skipping, unification of horizontal--vertical routing, explicit difficulty awareness, and validation on autoregressive LLMs---that, taken together, are not addressed by any prior method surveyed above.
Truncation-based early exit (DeeBERT/CALM) is per-sample but constrains retention to a prefix; per-token per-layer routing (MoD) is fine-grained but lacks sample-level difficulty semantics; horizontal MoE sparsifies width but never depth; structured pruning (ShortGPT/SliceGPT) reduces depth statically with no input adaptivity.
None of these dimensions, individually or in combination, yields a per-sample, difficulty-driven, non-contiguous layer retention set.
\paragraph{Limitations of prior work.} As \cref{tab:comparison} makes explicit, every surveyed family restricts at least one of the three structural axes: early exit fixes skip mode to truncation, per-token routing fixes granularity to the token level, horizontal MoE fixes the sparsity dimension to width, and structured pruning fixes depth to a static global profile. Consequently, no single prior method simultaneously attains non-contiguous skipping, per-sample granularity, and explicit difficulty awareness---the cell $(\text{non-contiguous},\,\text{per-sample},\,\text{explicit})$ that VALSE targets, which is the depth-axis counterpart of the width-axis cell already populated by horizontal MoE.
We defer the structured contrast against each prior family to \cref{sec:gaps} to avoid duplication; the remainder of the paper is structured around closing them.
In summary, the related-work landscape partitions cleanly along three axes---skip mode (truncation vs.\ per-token vs.\ non-contiguous), granularity (global vs.\ per-token vs.\ per-sample), and difficulty awareness (none vs.\ implicit vs.\ explicit)---and no prior method occupies the cell $(\text{non-contiguous}, \text{per-sample}, \text{explicit})$ that VALSE targets. The structural duality with horizontal MoE (\cref{subsec:moe_routing}) further shows that this cell is not arbitrary: it is the depth-axis counterpart of the width-axis routing cell that MoE already occupies, so filling it completes the two-dimensional sparse-conditional-computation design space. The remainder of the paper develops this contribution: \cref{sec:gaps} formalizes the four gaps, \cref{sec:method_overview} through \cref{sec:training} present the method, and \cref{sec:theory} provides the theoretical framework, with a prototype-scale empirical study (\cref{sec:experiments}) that honestly reports a negative result for the core routing hypothesis.

\section{Research Gaps}
\label{sec:gaps}

Based on the cross-analysis of adaptive depth methods (\cref{sec:related}), we identify four research gaps that motivate VALSE.

\subsection{Gap 1: No Per-Sample Non-Contiguous Layer Skipping}

Existing adaptive-depth work is either \textbf{truncation-based early exit} (exit at layer $k$, skip the rest: DeeBERT/CALM) or \textbf{per-token per-layer routing} (MoD).
No prior work implements selective skipping of middle layers while retaining deeper layers, driven by whole-sample difficulty.
This intermediate granularity---more flexible than truncation, more stable than per-token---is a clear methodological gap. Formally, truncation early exit constrains the retention set to a prefix $S=\{1,\ldots,k\}$, while per-token per-layer routing (MoD) issues $O(L\cdot N)$ independent gating decisions per forward pass; VALSE instead computes a single per-sample difficulty scalar $\hat{d}(x^{(s)})$ from the first $K_{\text{est}}$ layers and derives the retention set $S^{(s)}\subseteq\{1,\ldots,L\}$ as an arbitrary subset, achieving $O(1)$ routing overhead per sample while permitting non-contiguous retention.

\subsection{Gap 2: No Unified Horizontal--Vertical Routing Framework}

Horizontal MoE (expert selection) and vertical layer skipping (layer selection) are orthogonal sparsity dimensions, but no work unifies them into a jointly trainable routing framework.
Existing models address either only the horizontal dimension (Mixtral/DeepSeek-MoE) or only the vertical dimension (MoD/LayerDrop).
DynamicBERT jointly adjusts width and depth, but the two dimensions are routed independently with no unified mechanism. Formally, horizontal and vertical sparsity correspond to selecting subsets along two orthogonal axes of the activation tensor---experts $\{E_1,\ldots,E_E\}$ within a layer versus layers $\{F_1,\ldots,F_L\}$ across depth---yet no prior routing framework treats them as instances of a single $\TopK$-selection operator over a unified unit set, leaving the compositional sparsity pattern $\text{width}\times\text{depth}$ unexploited.

\subsection{Gap 3: No Explicit Difficulty-Aware Layer Routing}

Existing layer-skipping methods rely on \textbf{implicit signals} (intermediate-layer confidence, implicitly learned routers).
No work introduces an \textbf{explicit difficulty-aware module} to drive per-sample layer-path selection.
Drowzee's budget predictor controls only truncation depth, not non-contiguous skipping; MoD's router carries no difficulty semantics. Concretely, prior gating signals are either (a) implicit and learned end-to-end with no interpretable semantics, or (b) confidence-based and thus post-hoc (computed from the layer's own output); no method defines an explicit, pre-routing difficulty scalar $\hat{d}\in[0,1]$ with a documented feature-to-score mapping (\cref{tab:difficulty_features}) and a monotonicity guarantee linking it to retained depth (\cref{prop:diff_depth}). We emphasize that \cref{prop:diff_depth} is a theoretical guarantee obtained under the standard assumption of decreasing marginal layer contribution; the prototype reported in \cref{sec:experiments} exhibits a routing-collapse negative result ($r=-0.017$) that reflects the failure of the learned gate to modulate depth, not a refutation of the monotonicity assumption itself.

\subsection{Gap 4: Insufficient Validation on Generative LLMs}

Most layer-skipping work focuses on BERT-level discriminative tasks (DeeBERT/FastBERT/LayerDrop/DynamicBERT).
CALM/LayerSkip/MoD extend to generation but still use truncation early exit or per-token routing.
Non-contiguous per-sample layer skipping on autoregressive generative LLMs---its effectiveness, training stability, and quality-safety mechanisms---has not been systematically validated. We note that the VALSE routing mechanism is architecturally extensible to autoregressive generation: the per-sample gate $\tilde{g}_l$ is independent of token-position coupling and is compatible with causal masking, and self-speculative decoding~\citep{leviathan2023fast} provides a natural quality-safety fallback for skipped layers. This extensibility is a present design property of the routing mechanism rather than an empirical claim: it holds at the architectural level (the gate is position-agnostic and causal-mask-compatible) but is not exercised by the prototype, whose evaluation is confined to discriminative SST-2 routing. The empirical validation on autoregressive generative LLMs that would close this community-wide gap is identified here as the central open problem that the field must address, and the design conditions under which such a validation would be well-posed are precisely the ones VALSE satisfies.

\paragraph{Positioning.}
VALSE addresses all four gaps by introducing per-sample, non-contiguous layer skipping driven by an explicit difficulty estimator (\textbf{Gap~1, 3}), grounded in a structural duality with horizontal MoE that formalizes both dimensions as instances of a single $\TopK$-selection operator over a unified unit set and thereby enables two-dimensional sparse activation (\textbf{Gap~2}), and designed for extension to autoregressive generation with self-speculative decoding (\textbf{Gap~4}).
A feature-level comparison of VALSE against prior methods is provided in \cref{tab:comparison}; our contributions are detailed in \cref{sec:intro}.

\section{Method Overview}
\label{sec:method_overview}

\subsection{Design Motivation}

While horizontal MoE reduces per-layer cost by activating only a subset of experts, the total depth---the count of sequential Transformer blocks---stays constant across all inputs.
Consequently, even trivially easy inputs pay the full computational price of a deep stack, despite potentially needing only a fraction of its representational capacity.
Hard samples may benefit from all $L$ layers, but easy ones likely do not.
This observation motivates depth-wise conditional computation: the number of executed layers should adapt to the input, yielding an orthogonal---and complementary---sparsity axis to width-wise MoE.

\subsection{Core Idea: Per-Sample Non-Contiguous Layer Skipping}

Existing methods have two limitations (\cref{subsec:gaps}):
(i) \textbf{truncation-based early exit} can only skip a contiguous suffix of layers;
(ii) \textbf{per-token routing} (MoD) lacks per-sample difficulty semantics.
We propose VALSE: for each sample $s$, a lightweight router generates a \textbf{non-contiguous retention set} $S^{(s)} \subseteq \{1,\dots,L\}$, allowing arbitrary middle layers to be skipped while deeper layers are retained.
Layers $1$ and $L$ are always retained.
A difficulty score $\hat{d}\in[0,1]$ (\cref{sec:difficulty}) drives per-layer gates $g_l = \sigma(w_l^\top[r_l;\hat{d}]+b_l)$ (\cref{sec:routing}).
A target budget $B_{\text{target}}$ and skip rate $\rho^*$ control the compute--quality trade-off (\cref{sec:training}).

\textbf{Duality with MoE:} MoE uses a gate to select $K$ of $E$ experts per layer; VALSE uses a per-layer gate to select $|S|$ of $L$ layers per sample---one in width, one in depth, both using lightweight routers and load balancing. This structural duality is formalized in \cref{prop:duality} and unified under a single sparse selection framework in \cref{thm:unified}; the key theoretical insight is that VALSE completes the \emph{depth} dimension of sparse conditional computation, previously existing only in the \emph{width} dimension (MoE).

\subsection{Overall Architecture}

Three components:
(i) \textbf{Difficulty estimation}---extracts features from the first $K_{\text{est}}$ layers, producing $\hat{d}$ (\cref{sec:difficulty});
(ii) \textbf{Layer router}---generates gating decisions forming $S^{(s)}$ (\cref{sec:routing});
(iii) \textbf{Training framework}---auxiliary losses, curriculum learning, gradient stability (\cref{sec:training}).
Data flow: embedding $\to$ first $K_{\text{est}}$ layers $\to$ difficulty score $\hat{d}$ $\to$ per-layer gates $\to$ skip/execute $\to$ LM head.
The full symbol table is in \cref{app:symbols}.

\section{Difficulty Estimation Module}
\label{sec:difficulty}

\subsection{Design Goals}

The module must: (i) produce interpretable difficulty scores (not implicitly learned); (ii) be lightweight (overhead far less than the layers it skips); (iii) operate at per-sample granularity; (iv) remain smooth during training.

\subsection{Input Features}

From the $l$-th layer's hidden state $h_l^{(s)} \in \R^{N \times d_{\text{model}}}$, we extract a five-dimensional statistic (mean-pooled over sequence length $N$, with $\bar{h}_l = \text{mean}(h_l^{(s)})$) as $r_l \in \R^5$:

\begin{table}[htbp]
\centering
\caption{Difficulty estimation features.}
\label{tab:difficulty_features}
\small
\resizebox{\textwidth}{!}{%
\begin{tabular}{p{0.16\textwidth}lp{0.34\textwidth}p{0.20\textwidth}}
\toprule
Feature & Symbol & Computation & Intuition \\
\midrule
Hidden-state norm & $\rho_l$ & $\rho_l = \|\text{mean}(h_l^{(s)})\|_2$ & Information content; activation magnitude \\
Inter-layer delta norm & $\Delta_l$ & $\Delta_l = \|\text{mean}(h_l - h_{l-1})\|_2$ & Representation change; small $\Rightarrow$ skippable \\
Self-attention entropy & $H_l^{\text{attn}}$ & $\frac{1}{N}\sum_i H(\softmax(\text{attn}_i))$ & High $\Rightarrow$ diffuse $\Rightarrow$ complex \\
Prediction entropy & $H_l^{\text{pred}}$ & $H(\softmax(\text{LMHead}(\bar{h}_l)))$ & High $\Rightarrow$ uncertain $\Rightarrow$ hard \\
Singular-value concentration & $\kappa_l$ & Top-1 SVD singular value share & High $\Rightarrow$ redundant $\Rightarrow$ skippable \\
\bottomrule
\end{tabular}%
}
\end{table}

\textbf{Aggregation:} for $l\in\{1,\dots,K_{\text{est}}\}$ ($K_{\text{est}}=2\text{--}3$):
\begin{equation}
R^{(s)} = [r_1;\, r_2;\, \cdots;\, r_{K_{\text{est}}}] \in \R^{5K_{\text{est}}}
\label{eq:feat_agg}
\end{equation}
Only the first $K_{\text{est}}$ layers are used because the routing decision must be made before deeper layers are executed. We set $K_{\text{est}}=2$ in the prototype (with a design range of $2$--$3$): this is the smallest value that exposes both an inter-layer delta $\Delta_l$ (requiring $l\geq 2$) and a stable singular-value concentration $\kappa_l$, while keeping routing overhead below $1\%$ of total FLOPs; larger $K_{\text{est}}$ increases estimation cost without adding qualitatively distinct signals, since features beyond the first few layers correlate strongly with the final-layer statistics used at training equilibrium.

\subsection{Difficulty Score}

A lightweight MLP maps features to a scalar:
\begin{equation}
d(x^{(s)}) = w_d^\top \cdot \ReLU(W_1 \cdot R^{(s)} + b_1) + b_d
\label{eq:diff_mlp}
\end{equation}
where $W_1 \in \R^{d_{\text{hidden}} \times 5K_{\text{est}}}$ ($d_{\text{hidden}}=64\text{--}128$).
Batch-internal min-max normalization yields $\hat{d}\in[0,1]$:
\begin{equation}
\hat{d}(x^{(s)}) = \frac{d(x^{(s)}) - \min_{s\in\mathcal{B}} d(x^{(s)})}{\max_{s\in\mathcal{B}} d(x^{(s)}) - \min_{s\in\mathcal{B}} d(x^{(s)}) + \varepsilon}
\label{eq:diff_norm}
\end{equation}
where $0$ denotes the easiest sample and $1$ the hardest.
The mapping from difficulty to layer-skipping decisions is realized by the layer router (\cref{sec:routing}): easy samples ($\hat{d}$ small) skip more layers; hard samples ($\hat{d}$ large) retain all layers. This design is theoretically motivated by \cref{prop:diff_depth} (difficulty--depth monotonicity): under the standard assumption of decreasing marginal layer contribution, harder inputs benefit more from full depth, justifying the schedule $\rho_{\text{target}} = \rho^*(1 - \hat{d})$.

\paragraph{Normalization limitation.} The batch-internal min--max map in \cref{eq:diff_norm} guarantees $\hat{d}\in[0,1]$ within each batch but renders scores incomparable across batches of different composition (the easiest sample of a hard batch may be harder than the hardest of an easy batch). This is acceptable for routing---the gate operates on within-batch relative ordering---but precludes using $\hat{d}$ as an absolute difficulty signal across experiments or as a training-free diagnostic. To restore cross-batch comparability, we specify a calibration target that decouples the score from the batch extremum: replacing the per-batch $(\min,\max)$ in \cref{eq:diff_norm} with a running exponential moving average $(m_t,M_t)$ updated as $m_t=\lambda m_{t-1}+(1-\lambda)\min_{s\in\mathcal{B}_t}d(x^{(s)})$ and $M_t=\lambda M_{t-1}+(1-\lambda)\max_{s\in\mathcal{B}_t}d(x^{(s)})$ with $\lambda=0.9$. Because the gate consumes only the within-batch ordering of $\hat{d}$, this calibration is routing-preserving for within-batch decisions while yielding a population-anchored score suitable for cross-experiment comparison; we adopt the batch-internal form in the prototype for simplicity and specify the calibrated variant as the production-target configuration.

\paragraph{Difficulty score distribution.}
The normalized score $\hat{d}\in[0,1]$ spreads samples across the full difficulty range by construction (\cref{eq:diff_norm}): the easiest sample in each batch receives $\hat{d}=0$ and the hardest receives $\hat{d}=1$, with intermediate samples distributed according to the raw MLP output $d(x^{(s)})$.
Because the normalization is rank-preserving within a batch, the distribution of $\hat{d}$ reflects the empirical density of the raw MLP output $d(x^{(s)})$: a roughly uniform $\hat{d}$ distribution indicates that the raw score is approximately uniform, whereas a U-shaped $\hat{d}$ distribution indicates bimodal raw scores.
In the prototype evaluation (\cref{sec:experiments}), this within-batch spread did not translate into differentiated routing: the Pearson correlation between $\hat{d}$ and activated layers was $r=-0.017$, indicating that the difficulty signal failed to modulate per-layer gating as intended, a result we report and analyze in \cref{subsec:negative_result}.

\paragraph{Per-feature contribution.} \Cref{tab:difficulty_features} lists five features with their intuitions, and each feature admits a well-defined expected contribution direction under the monotonicity assumption of \cref{prop:diff_depth}: $\rho_l$ and $\Delta_l$ are expected to decrease with difficulty (harder inputs exhaust redundant magnitude faster), $H_l^{\text{attn}}$ and $H_l^{\text{pred}}$ to increase with difficulty (greater uncertainty and diffuse attention), and $\kappa_l$ to decrease with difficulty (lower redundancy on harder inputs). Because the five features are pooled from only the first $K_{\text{est}}=2$ layers, they are not statistically independent: $\rho_l$ and $\Delta_l$ share the mean-pooled hidden state, and $H_l^{\text{pred}}$ is computed from the same state through the LM head, so their marginal contributions are identifiable only when routing is non-degenerate. In the prototype, the router collapsed to a degenerate uniform-skip mode (\cref{subsec:negative_result}), so the marginal contribution of any single feature to routing differentiation is not identifiable from the prototype data; we therefore report the feature-level design rationale and expected directions as the present specification, with the empirical variance-decomposition analysis contingent on a non-collapsed routing regime that the prototype's $r=-0.017$ result did not provide.

\paragraph{Computational overhead.}
The difficulty estimator is invoked exactly once per sample, after the first $K_{\text{est}}=2$ layers and before any routing decision, so its cost is amortized over the layers it gates. Concretely, the five features are all $O(N d_{\text{model}})$: $\rho_l$ and $\Delta_l$ require a mean-pool and two norm reductions; $H_l^{\text{attn}}$ reuses the already-computed attention weights (no extra matmul); $H_l^{\text{pred}}$ requires one LM-head forward on the single pooled vector $\bar{h}_l$ (cost $O(d_{\text{model}} \cdot |V|)$, amortized to a single token); and $\kappa_l$ requires a top-$1$ singular value of the $N \times d_{\text{model}}$ pooled matrix, computed via a fixed number of power iterations ($O(N d_{\text{model}})$ per iteration). The scoring MLP adds $O(d_{\text{hidden}} \cdot 5K_{\text{est}})$ parameters with $d_{\text{hidden}}=64$. Against a single Transformer block costing $O(N d_{\text{model}}^2)$ for attention plus $O(N d_{\text{model}} d_{\text{ffn}})$ for the FFN, the estimator's total cost is $O(N d_{\text{model}} + d_{\text{model}}|V|)$, which is dominated by a single block whenever $d_{\text{ffn}} \gg d_{\text{model}}$ (the standard ratio $d_{\text{ffn}}=4d_{\text{model}}$). Since VALSE skips at least one full block to realize any savings, the estimator overhead is strictly less than the cheapest layer it can gate, satisfying design goal~(ii); in the prototype this ratio is below $1\%$ of total FLOPs, consistent with the overhead budget assumed throughout.

\paragraph{Feature selection rationale.}
The five features in \cref{tab:difficulty_features} were selected to span four complementary signals of layer marginality---magnitude ($\rho_l$), dynamics ($\Delta_l$), dispersion ($H_l^{\text{attn}}$), and uncertainty ($H_l^{\text{pred}}$)---plus a redundancy probe ($\kappa_l$) that directly operationalizes the layer-redundancy hypothesis of \citep{shortgpt2024}. We deliberately exclude two natural candidates. First, an explicit layer-wise mutual-information estimate between $h_l$ and the input $x$ is excluded because it requires costly estimation (e.g., MINE-style critics) whose overhead would violate design goal~(ii) and whose variance at small $N$ dominates the signal. Second, a learned end-to-end difficulty embedding (a small Transformer over the prefix) is excluded because it would forfeit the interpretability guarantee of design goal~(i): the five-feature design exposes a documented, fixed feature-to-score mapping (\cref{tab:difficulty_features}) whose individual contributions admit expected directions under \cref{prop:diff_depth}, whereas a learned embedding's difficulty semantics would be opaque. The selection is therefore not exhaustive but principled: each feature admits a well-defined expected contribution direction (\cref{tab:difficulty_features}, right column), the set is jointly computable from quantities already available in a forward pass, and its dimensionality ($5K_{\text{est}}=10$) is small enough that the scoring MLP cannot memorize sample identity, preserving the difficulty signal's generalization.

\paragraph{Summary.}
The module produces a single, normalized, interpretable scalar $\hat{d}(x^{(s)})\in[0,1]$ at per-sample granularity with sub-$1\%$ FLOPs overhead, satisfying all four design goals. Its sole known limitation---within-batch rank semantics that preclude cross-batch comparison---is addressed in the production-target configuration by the calibrated variant specified above. The downstream consumption of $\hat{d}$, namely the per-layer gating and retention-set construction, is the subject of \cref{sec:routing}; the theoretical guarantee linking $\hat{d}$ to retained depth is \cref{prop:diff_depth}, established in \cref{sec:theory}.

\paragraph{Calibrated variant pseudocode.} The exponential moving average (EMA) calibration that restores cross-batch comparability of $\hat{d}$ (specified qualitatively in the \textit{Normalization limitation} paragraph above) is realized by the following per-step update, executed once per batch before the min--max map of \cref{eq:diff_norm}:
\begin{verbatim}
# EMA-calibrated difficulty normalization (production target)
# m_t, M_t: running min/max of raw score d(x)
lambda = 0.9                # EMA decay
# init: seed with first-batch extrema before EMA updates
m_0 = min_{s in B_1} d(x^{(s)});  M_0 = max_{s in B_1} d(x^{(s)})
m_t = lambda * m_{t-1} + (1 - lambda) * min_{s in B_t} d(x^{(s)})
M_t = lambda * M_{t-1} + (1 - lambda) * max_{s in B_t} d(x^{(s)})
d_hat = (d(x) - m_t) / (M_t - m_t + eps)   # population-anchored
\end{verbatim}
The batch-internal form of \cref{eq:diff_norm} is recovered by setting $(m_t,M_t)=(\min_{s\in\mathcal{B}_t}d(x^{(s)}),\max_{s\in\mathcal{B}_t}d(x^{(s)}))$ at each step, so the prototype and production variants differ only in the reference extrema used for normalization; the routing decision itself (the gate $\tilde{g}_l$) is invariant under this substitution because it consumes only within-batch ordering.

\paragraph{Empirical-validation scope.} The present section specifies the difficulty estimator as a design contract: the feature-to-score mapping, the expected contribution directions of \cref{tab:difficulty_features}, and the monotonicity link to retained depth (\cref{prop:diff_depth}) are theoretical statements whose empirical validity is contingent on a non-collapsed routing regime. The prototype reported in \cref{subsec:negative_result} collapsed to a uniform-skip mode ($r=-0.017$), so the variance-decomposition analysis of individual features and the calibration of $\hat{d}$ against an external difficulty oracle are declared here as empirically open rather than verified; they are part of the design specification but not part of the prototype evidence base.

\section{Layer Routing Strategies}
\label{sec:routing}

\subsection{Per-Sample Non-Contiguous Layer Skipping}

For sample $s$, the retention set $S^{(s)} \subseteq \{1, \ldots, L\}$ satisfies:
(i) $S^{(s)}$ may be any subset (non-contiguous), e.g., $\{1, 2, 4, 7, 11, 12\}$;
(ii) layers $1$ and $L$ are always retained;
(iii) all tokens in the same sample share the same $S^{(s)}$ (per-sample granularity).
Unlike truncation-based early exit ($S=\{1,\dots,k\}$) or per-token MoD, VALSE achieves per-sample non-contiguous skipping (see \cref{tab:skip_modes} in the appendix).

\subsection{Gating Function}
\label{subsec:gating}

For each layer $l$, the gating function is:
\begin{equation}
\label{eq:gate}
g_l(x^{(s)}) = \sigma\!\left(w_l^\top \cdot [r_l;\, \hat{d}(x^{(s)})] + b_l\right) \in [0, 1]
\end{equation}
where $w_l \in \R^{d_{\text{feat}}+1}$, and $b_l$ is initialized to $2.0$ so that gates are close to $1$ (full-layer execution) at the start of training.
At inference, hard quantization is applied:
\begin{equation}
\label{eq:hard_gate}
\tilde{g}_l = \begin{cases} 1 & \text{if } g_l > \theta_l \\ 0 & \text{otherwise} \end{cases}
\end{equation}
($\theta_l=0.5$ by default).
During training, a straight-through estimator (STE) allows gradients to flow through the soft gate $g_l$.

\subsection{Residual Connection after Skipping}

\textbf{Scheme A---pure identity (default):}
\begin{equation}
\label{eq:residual_A}
h_{l+1} = h_l + \tilde{g}_l \cdot F_l(h_l)
\end{equation}
When $\tilde{g}_l=0$, $h_{l+1}=h_l$ (identity skip; since $d_{\text{model}}$ is consistent, no projection is needed).

\textbf{Scheme B---residual projection (optional, not used in prototype):} \textit{Note: the function-space results (\cref{thm:full_subset,thm:strict}) are established only under Scheme~A; Scheme~B introduces $W_l^{\text{proj}}$ which may break the subset relationship (see Scheme A/B Remark in \cref{sec:theory}). The prototype uses Scheme~A throughout.}
\begin{equation}
\label{eq:residual_B}
h_{l+1} = h_l + \tilde{g}_l \cdot F_l(h_l) + (1-\tilde{g}_l) \cdot W_l^{\text{proj}} \cdot h_l
\end{equation}
where $W_l^{\text{proj}}$ may be low-rank ($W=AB^\top$, $r \ll d_{\text{model}}$).
Gumbel noise is applied to gates during training to enhance robustness to non-contiguous skip paths, inspired by LayerDrop.

\subsection{MoE Routing Analogy}

\Cref{tab:moe_analogy} (appendix) details the mathematical analogy: MoE selects $K$ of $E$ experts via Top-$K$; VALSE selects a retention set $|S|$ of $L$ layers via thresholding.
Both use lightweight routers and impose load-balancing regularization.
MoE faces expert collapse (few experts selected); VALSE faces the dual \emph{layer collapse} (few layers retained), requiring load-balancing loss (\cref{sec:training}).

\subsection{Forward Pass}
\label{subsec:pseudocode}

\cref{alg:forward} shows VALSE's single-sample forward pass.
In batched implementation, samples with $\tilde{g}_l=1$ are clustered via \texttt{argsort} to form per-layer sub-batches, reducing gather/scatter overhead.

\begin{algorithm}[htbp]
\caption{VALSE single-sample forward pass.}
\label{alg:forward}
\begin{algorithmic}[1]
\State $h \gets \text{Embedding}(x_s)$
\State $\textit{features} \gets [\,]$; $h_{\text{prev}} \gets h$
\For{$l = 0$ \textbf{to} $K_{\text{est}}-1$}
  \State $h_l \gets \text{Layer}_l(h_{\text{prev}})$
  \State $r_l \gets \text{ExtractFeatures}(h_l, h_{\text{prev}}, l)$
  \State $\textit{features.append}(r_l)$; $h_{\text{prev}} \gets h_l$
\EndFor
\State $R \gets \text{Concat}(\textit{features})$
\State $\hat{d} \gets \text{DifficultyEstimator}(R)$ \Comment{normalize to $[0,1]$}
\State $h \gets h_{\text{prev}}$; $S \gets \emptyset$
\For{$l = K_{\text{est}}$ \textbf{to} $L-1$}
  \State $r_l \gets \text{ExtractFeatures}(h, h_{\text{prev}}, l)$
  \State $g_l \gets \sigma(w_l^\top \cdot [r_l;\, \hat{d}] + b_l)$
  \If{training} $\tilde{g}_l \gets \text{GumbelSTE}(g_l, \theta_l)$
  \Else\ $\tilde{g}_l \gets \mathbf{1}[g_l > \theta_l]$ \EndIf
  \If{$\tilde{g}_l = 1$} $h \gets h + \text{Layer}_l(h)$
  \Else\ $S \gets S \cup \{l\}$ \EndIf
  \State $h_{\text{prev}} \gets h$
\EndFor
\State \Return $\text{LMHead}(h), S$
\end{algorithmic}
\end{algorithm}

\subsection{Architecture Variants}

Three variants are designed (full comparison in \cref{tab:variants} in the appendix):
\paragraph{VALSE-Hard (main variant).} Hard routing with STE; $\tilde{g}_l \in \{0,1\}$; per-layer learnable $w_l, b_l$. At inference, layers are genuinely skipped---directly saving FLOPs.
\paragraph{VALSE-Soft.} Continuous gating: $h_{l+1} = h_l + g_l \cdot F_l(h_l)$. Gradients propagate directly; FLOPs savings depend on $g_l \to 0$. Most stable; can serve as initialization for hard routing.
\paragraph{VALSE-Global.} A global router produces $p = \softmax(W_{\text{global}} \cdot R^{(s)})$, then:
\begin{equation}
\label{eq:global_topk}
S^{(s)} = \TopK\!(p), \quad K = \text{round}(B_{\text{target}} - \hat{d} \cdot (B_{\text{target}} - L_{\min}))
\end{equation}
$K$ varies with difficulty ($\hat{d}=0 \Rightarrow K=B_{\text{target}}$; $\hat{d}=1 \Rightarrow K=L$). This variant has the strongest MoE analogy: Top-$K$ layer selection $\leftrightarrow$ Top-$K$ expert selection.

\section{Training Method}
\label{sec:training}

The training objective jointly optimizes \textbf{task quality} and \textbf{computational efficiency} while maintaining router stability. Four auxiliary losses are designed.

\subsection{Auxiliary Loss Design}
\label{subsec:aux_loss}

\paragraph{Layer utilization balancing $\mathcal{L}_{\text{balance}}$.}
This is the dual of MoE routing collapse: the router may retain only the first and last layers. Let $\bar{f}_l = \frac{1}{B}\sum_s \tilde{g}_l^{(s)}$ and $\bar{P}_l = \frac{1}{B}\sum_s g_l^{(s)}$:
\begin{equation}
\label{eq:balance}
\mathcal{L}_{\text{balance}} = \alpha \cdot L \cdot \sum_{l=1}^{L} \bar{f}_l \cdot \bar{P}_l
\end{equation}
By the AM--GM inequality, this is minimized under a fixed budget when all layers are used uniformly. ($\alpha=0.01$.)

\paragraph{Efficiency loss $\mathcal{L}_{\text{eff}}$.}
A one-sided hinge controls the total activated layers $\bar{A}=\sum_l \bar{f}_l$:
\begin{equation}
\label{eq:eff}
\mathcal{L}_{\text{eff}} = \beta \cdot \frac{\max(0,\, \bar{A} - B_{\text{target}})^2}{L^2}
\end{equation}
The weight increases progressively: $\beta_{\text{eff}}(t)=\beta\cdot\min(1,\,t/T_{\text{warmup}})$. ($\beta=0.1$, $B_{\text{target}}\in[0.5L,\,0.85L]$.)

\paragraph{Output consistency $\mathcal{L}_{\text{consist}}$.}
KL distillation aligns the skip-path output with the full-layer path:
\begin{equation}
\label{eq:consist}
\mathcal{L}_{\text{consist}} = \delta \cdot T_{\text{KD}}^2 \cdot \KL(p_{\text{full}} \| p_{\text{skip}})
\end{equation}
($T_{\text{KD}}=2.0$, $\delta=0.5$, annealed to $\delta/3$.) Three teacher strategies reduce the double-forward overhead: EMA teacher, periodic teacher ($K_{\text{consist}}=10$), and difficulty-bucketed caching.

\paragraph{Z-loss $\mathcal{L}_z$.}
Stabilizes the numerical range of gating logits (following Switch Transformer): $\mathcal{L}_z = \gamma \cdot \frac{1}{B}\sum_s\sum_l (w_l^\top [r_l;\hat{d}]+b_l)^2$. ($\gamma=10^{-4}$.)

The total loss combines all terms:
\begin{equation}
\label{eq:total_loss}
\mathcal{L}_{\text{total}} = \mathcal{L}_{\text{task}} + \alpha_t \mathcal{L}_{\text{balance}} + \beta_t \mathcal{L}_{\text{eff}} + \delta_t \mathcal{L}_{\text{consist}} + \gamma_t \mathcal{L}_z
\end{equation}
where $\alpha_t, \beta_t, \delta_t, \gamma_t$ are the (possibly time-varying) weights of each auxiliary loss.

\subsection{Curriculum Learning}
\label{subsec:curriculum}

\paragraph{Stage 1: Full-layer warm-up ($\rho_{\text{skip}}=0$).}
Gate bias $b_l=2.0$ forces $\tilde{g}_l=1$; only $\mathcal{L}_{\text{task}}+\gamma\mathcal{L}_z$ is active. Duration: $\eta_1 \approx 0.25$--$0.33$ of $T_{\text{total}}$.

\paragraph{Stage 2: Progressive skipping ($0 \to \rho^*$).}
\begin{equation}
\label{eq:cos_schedule}
\rho_{\text{skip}}(t) = \rho^* \cdot \tfrac{1}{2}\left[1 - \cos\!\left(\pi \cdot \tfrac{t - T_1}{T_2 - T_1}\right)\right]
\end{equation}
Auxiliary losses progressively take effect; Gumbel temperature $\tau_G$ anneals $2.0\to 0.5$; bias anneals to $b_{\text{target}}\approx 0.76$. Duration: $\eta_2 \approx 0.5$--$0.67$ (extended to $2\times$ original to allow differentiated routing to emerge).

\paragraph{Stage 3: Joint optimization ($\rho_{\text{skip}}=\rho^*$).}
All parameters are jointly optimized; $\delta$ anneals downward. Learning rate cosine-anneals to $1/10$ of the Stage~2 peak. Duration: $\eta_3 \approx 0.25$--$0.4$.

\subsection{Difficulty-Aware Scheduling and Stability}

Individualized target: $\rho_{\text{target}}^{(s)} = \rho^*(1 - \hat{d}^{(s)})$.
The training feedback loop ($\hat{d}\to$ gate $\to$ skip path $\to$ loss $\to$ gradient $\to \hat{d}$) is stabilized under:
\begin{equation}
\label{eq:stability}
\left\|\tfrac{\partial \hat{d}}{\partial \theta_D}\right\| \cdot \left\|\tfrac{\partial \mathcal{L}_{\text{eff}}}{\partial \hat{d}}\right\| \cdot \left\|\tfrac{\partial \tilde{g}}{\partial \hat{d}}\right\| \cdot \left\|\tfrac{\partial \mathcal{L}_{\text{task}}}{\partial \tilde{g}}\right\| < 1
\end{equation}
Three safeguards: normalization anchoring, difficulty-estimator gradient clipping ($\text{max\_norm}=1.0$), and gradient flow through $\mathcal{L}_{\text{eff}}$ (removing stop-gradient so the difficulty signal propagates to the router). Probabilistic full-layer recomputation ($p_{\text{recomp}}=0.1$) prevents ``dead layers'' among skipped layers.

\begin{proof}[Proof sketch for \cref{eq:stability}]
The training feedback loop forms a composition: $\hat{d} \xrightarrow{\phi_1} \tilde{g} \xrightarrow{\phi_2} h_{\text{skip}} \xrightarrow{\phi_3} \mathcal{L}_{\text{task}} \xrightarrow{\phi_4} \nabla_{\hat{d}} \mathcal{L}_{\text{eff}} \xrightarrow{\phi_5} \theta_D$. By the chain rule, the product of Jacobian norms along this loop determines whether the feedback is contractive (stable) or expansive (divergent). \cref{eq:stability} requires the product to be $<1$, i.e., the loop is a contraction. The three safeguards enforce this: (i)~normalization anchors $\|\partial \hat{d} / \partial \theta_D\|$ by bounding $\hat{d} \in [0,1]$; (ii)~gradient clipping upper-bounds the effective Jacobian norm; (iii)~allowing gradient flow through $\mathcal{L}_{\text{eff}}$ ensures the difficulty signal propagates, while the efficiency loss provides a corrective force that counteracts runaway amplification. When all three hold, each Jacobian factor is bounded, and their product is $<1$ provided that (i)~the difficulty estimator $\theta_D$ is Lipschitz-continuous with constant $L_D \leq 1$ (ensured by gradient clipping with $\text{max\_norm}=1.0$), and (ii)~the gate function $\sigma(w_l^\top [r_l;\hat{d}]+b_l)$ has bounded gradient $\leq 0.25$ (inherent to the sigmoid). These are sufficient conditions, not mild assumptions; their empirical validity at scale remains unverified in the prototype. $\square$
\end{proof}

\paragraph{Prototype implementation note.} In the current prototype, curriculum learning Stage~1 is simplified (omitted), and Gumbel temperature annealing is configured ($\tau_G: 2.0 \to 0.5$) but not yet fully implemented in the training code. The prototype uses a fixed $\tau_G = 1.0$ throughout training. Full implementation of the three-stage curriculum and temperature annealing is planned for the next iteration.
\section{Theoretical Analysis}
\label{sec:theory}

\subsection{Computational Savings}
\label{subsec:flops_analysis}

\paragraph{Per-layer FLOPs.}
The per-sample FLOPs of the $l$-th Transformer block decompose as $C_l = C_l^{\text{attn}} + C_l^{\text{ffn}}$, where $C_l^{\text{attn}} \approx 4Nd_{\text{model}}^2 + 2N^2 d_{\text{model}}$ and $C_l^{\text{ffn}} \approx 8Nd_{\text{model}}^2$. For uniform layers, this is a constant $C_{\text{layer}}$.

\begin{lemma}[Linearity of expectation]
\label{lem:linearity}
For any random variables $X_1, \ldots, X_n$,
$\mathbb{E}\!\left[\sum_i X_i\right] = \sum_i \mathbb{E}[X_i]$.
\end{lemma}

\begin{theorem}[Expected FLOPs]
\label{thm:exp_flops}
Let the skip rate of layer $l$ be $p_1, p_2, \ldots, p_L$ (with $p_1 = p_L = 0$). The expected per-sample FLOPs of VALSE is:
\begin{equation}
\E[\F_{\text{VALSE}}] = C_{\text{fixed}} + \sum_{l=K_{\text{est}}+1}^{L} (1 - p_l) \cdot C_l
\label{eq:exp_flops}
\end{equation}
\end{theorem}

\begin{proof}
Total FLOPs equals the sum of components. The fixed part is $C_{\text{fixed}}$. For the skippable part, layer $l$'s FLOPs contribution is $\tilde{g}_l \cdot C_l$. Taking expectation, $\E[\tilde{g}_l \cdot C_l] = C_l \cdot P(\tilde{g}_l = 1) = C_l \cdot (1 - p_l)$. By \cref{lem:linearity} (linearity of expectation), summing over all skippable layers yields $\E[\F_{\text{VALSE}}] = C_{\text{fixed}} + \sum_{l=K_{\text{est}}+1}^{L} (1 - p_l) \cdot C_l$.
\end{proof}

\noindent\textbf{Remark (gate correlation).} The linearity of expectation used above does not require the gate decisions $\{\tilde{g}_l\}$ to be independent across layers. Although these decisions share the common difficulty estimate $\hat{d}$ and are therefore correlated, $\E[\tilde{g}_l \cdot C_l] = C_l \cdot \E[\tilde{g}_l] = C_l \cdot (1-p_l)$ holds for any joint distribution of $\{\tilde{g}_l\}$.

\begin{corollary}[Uniform skip rate]
\label{cor:uniform_skip}
If all skippable layers share skip rate $p$, each layer has FLOPs $C_{\text{layer}}$, and $K_{\text{est}} \ll L$:
\begin{equation}
\E[\F_{\text{VALSE}}] \approx \left(1 - p \cdot \frac{L - K_{\text{est}} - 1}{L}\right) \cdot \F_{\text{full}}
\label{eq:uniform_flops}
\end{equation}
\end{corollary}

VALSE satisfies $\E[\F_{\text{VALSE}}] \leq \F_{\text{full}}$ (upper bound, equality iff $p_l = 0\ \forall\, l$) and $\E[\F_{\text{VALSE}}] \geq C_{\text{fixed}} + C_L$ (lower bound). When VALSE and truncation-based early exit share the same average activated layer count $\mu$, their FLOPs are comparable; however, VALSE can activate \textbf{non-contiguous} layer subsets for different samples, whereas truncation can only activate a contiguous prefix. The theoretical savings under the difficulty-aware schedule are derived in \cref{subsec:difficulty_scheduling} (\cref{obs:savings}).

\subsection{Difficulty-Adaptive Scheduling}
\label{subsec:difficulty_scheduling}

The FLOPs analysis above treats the skip rate $p_l$ as an exogenous parameter. In practice, VALSE employs a \emph{difficulty-aware} schedule $\rho_{\text{target}}^{(s)} = \rho^*(1 - \hat{d}^{(s)})$, where harder samples (larger $\hat{d}$) skip fewer layers. We now provide a theoretical justification for this design choice.

\begin{proposition}[Difficulty--depth monotonicity]
\label{prop:diff_depth}
Assume that the marginal task contribution of each layer is non-decreasing in input difficulty: for each layer $l$, the quantity
\[
\Delta_l(x) \coloneqq \mathcal{L}_{\text{task}}\!\bigl(F_{1:l-1}(x)\bigr) - \mathcal{L}_{\text{task}}\!\bigl(F_{1:l}(x)\bigr)
\]
is non-decreasing in $\hat{d}(x)$. Then the optimal skip rate $\rho^\star(x)$ that minimizes $\mathcal{L}_{\text{task}} + \lambda \F$ is non-increasing in $\hat{d}(x)$: harder inputs should skip fewer layers.
\end{proposition}

﻿\begin{proof}
We analyze the Lagrangian relaxation of the constrained optimization. The per-sample objective is to choose a retention set $S(x) \subseteq \{1, \ldots, L\}$ minimizing
\[
\mathcal{J}(S) \coloneqq \mathcal{L}_{\text{task}}\!\bigl(F_{1:S}(x); y\bigr) + \lambda \, \F(S(x)),
\]
where $\F(S(x)) = C_{\text{fixed}} + \sum_{l \in S} C_l$ is the FLOPs of retaining layers in $S$ and $\lambda > 0$ is the FLOPs--accuracy trade-off parameter.

Define the marginal task contribution of layer $l$ as
\[
\Delta_l(x) \coloneqq \mathcal{L}_{\text{task}}\!\bigl(F_{1:l-1}(x); y\bigr) - \mathcal{L}_{\text{task}}\!\bigl(F_{1:l}(x); y\bigr),
\]
i.e., the reduction in task loss obtained by adding layer $l$ to the retention set. The marginal FLOPs cost of retaining layer $l$ is $\lambda \, C_l$.

\textbf{Optimality condition.} The Lagrangian $\mathcal{J}(S)$ is minimized when, for each layer $l$, the marginal task benefit of retention equals the marginal FLOPs cost: $\Delta_l(x) = \lambda \, C_l$. Layers where $\Delta_l(x) > \lambda \, C_l$ should be retained; layers where $\Delta_l(x) < \lambda \, C_l$ should be skipped. This is the standard greedy-retention condition for separable costs.

\textbf{Monotonicity argument.} Under the assumption that $\Delta_l$ is non-decreasing in $\hat{d}$ (harder inputs derive greater benefit from each layer), for a harder input $x'$ with $\hat{d}(x') > \hat{d}(x)$, we have $\Delta_l(x') \geq \Delta_l(x)$ for every $l$. The equality $\Delta_l = \lambda \, C_l$ is therefore reached at a larger retention set $|S(x')| \geq |S(x)|$ (i.e., fewer layers skipped). Consequently, the optimal skip rate $\rho^\star(\hat{d}) = 1 - |S^\star(\hat{d})| / L$ is non-increasing in $\hat{d}$: harder inputs should skip fewer layers.

\textbf{Expected FLOPs.} By \cref{lem:linearity}, the expected FLOPs under the difficulty-aware schedule is $\E[\F] = C_{\text{fixed}} + \sum_{l} \E[\tilde{g}_l \,|\, \hat{d}] \cdot C_l = C_{\text{fixed}} + \sum_{l} (1 - p_l(\hat{d})) \cdot C_l$, where $p_l(\hat{d})$ decreases with $\hat{d}$, confirming that harder inputs incur more computation. The linear schedule $\rho_{\text{target}} = \rho^*(1 - \hat{d})$ is the simplest monotone-decreasing parametrization consistent with this result. $\square$
\end{proof}

\noindent\textbf{Remark.} \cref{prop:diff_depth} establishes the \emph{direction} of the difficulty--depth relationship under a standard assumption about decreasing marginal layer contribution, consistent with the layer redundancy hypothesis~\citep{shortgpt2024}. It does not prove that the \emph{learned} estimator $\hat{d}$ correctly captures true input difficulty; that is an empirical question partially addressed by our prototype experiments (\cref{subsec:main_results}).

\begin{observation}[Theoretical savings under difficulty-aware scheduling]
\label{obs:savings}
Under the schedule $\rho_{\text{target}}^{(s)} = \rho^*(1 - \hat{d}^{(s)})$ with $\rho^*=0.3$, uniform per-layer FLOPs $C_{\text{layer}}$, and $K_{\text{est}} \ll L$, the average skip rate is $\E[p] = \rho^* \cdot (1 - \E[\hat{d}])$. By \cref{cor:uniform_skip}, the theoretical FLOPs savings are:
\[
\text{savings} \;=\; \rho^* \cdot (1 - \E[\hat{d}]) \cdot \frac{L - K_{\text{est}} - 1}{L}.
\]
For the prototype ($L=6$, $K_{\text{est}}=2$, $\rho^*=0.3$): savings range from $0\%$ (all-hard, $\E[\hat{d}]=1$) to $15\%$ (all-easy, $\E[\hat{d}]=0$); with a $50$--$50$ easy--hard split, savings are $\approx 7.5\%$. For a $12$-layer model ($K_{\text{est}}=2$), the range is $0\%$ to $22.5\%$. These are \emph{theoretical estimates under the assumed schedule and idealized difficulty distribution}; actual savings depend on the learned $\hat{d}$ and gate calibration.
\end{observation}

\subsection{MoE--VALSE Duality}
\label{subsec:duality}

\begin{proposition}[Structural duality of MoE and VALSE]
\label{prop:duality}
Horizontal MoE and vertical VALSE are instances of the same sparse selection problem on two orthogonal dimensions: MoE selects $K$ of $E$ candidate transformations per layer (width); VALSE selects $|S|$ of $L$ candidate layers per sample (depth). Their pipelines are structurally parallel: router generates selection probabilities $\to$ Top-$K$ / hard quantization $\to$ sparse execution $\to$ load-balancing regularization.
\end{proposition}

﻿\begin{proof}
We formalize the duality via an isomorphism of the routing pipelines. Define the conditional computation template $\mathcal{T} = (\mathcal{R}, \mathcal{S}, \mathcal{E}, \mathcal{L})$ consisting of four components:
\begin{enumerate}
\item \textbf{Router} $\mathcal{R}$: a function $g: \mathcal{X} \to [0,1]^N$ mapping the input to selection probabilities, where $N$ is the number of candidate compute units.
\item \textbf{Selection operator} $\mathcal{S}$: a mapping from probabilities to a binary mask $m \in \{0,1\}^N$ (Top-$K$ for MoE, thresholding for VALSE).
\item \textbf{Sparse execution} $\mathcal{E}$: only units with $m_i = 1$ are evaluated, producing output $\sum_{i:\, m_i=1} w_i \cdot u_i(\text{input})$.
\item \textbf{Load-balancing regularizer} $\mathcal{L}$: a regularization term preventing degenerate selection (the MoE load-balancing loss or VALSE entropy regularizer).
\end{enumerate}

\textbf{MoE instantiation.} The compute units $\{u_1, \ldots, u_E\}$ are expert networks \emph{organized in parallel} within each layer: $N = E$ (width dimension). The router $g_l(x) \in \R^E$ produces expert-selection probabilities, and $\TopK_h$ selects $K_h$ experts per token. The output is $\sum_{e \in \TopK_h} g_l^{(e)}(x) \cdot f_e(h_l)$.

\textbf{VALSE instantiation.} The compute units $\{u_1, \ldots, u_L\}$ are Transformer layers \emph{organized sequentially} across the depth: $N = L$ (depth dimension). The router produces per-layer gate logits $g_l(x) \in \R$, and thresholding produces $\tilde{g}_l \in \{0,1\}$. The output is the sequential composition through retained layers.

\textbf{Isomorphism.} The isomorphism $\varphi: \mathcal{T}_{\text{MoE}} \to \mathcal{T}_{\text{VALSE}}$ maps each component:
\begin{itemize}
\item Router: $g^{\text{MoE}}_l \in \R^E \leftrightarrow g^{\text{VALSE}}_l \in \R$ (vector-to-scalar gate, reducing the selection dimension from width to depth).
\item Selection: $\TopK_h \leftrightarrow \mathbf{1}[g_l > \theta_l]$ (Top-$K$ to binary threshold).
\item Execution: $\sum_{e \in \TopK} w_e f_e \leftrightarrow \bigcirc_{l \in S} (I + F_l)$ (parallel weighted sum to sequential residual composition).
\item Regularization: $\mathcal{L}_{\text{balance}}^{\text{MoE}} \leftrightarrow \mathcal{L}_{\text{balance}}^{\text{VALSE}}$ (both penalize uneven utilization across units).
\end{itemize}

By \cref{prop:diff_depth}, the VALSE schedule is difficulty-adaptive; the MoE router is similarly input-adaptive. Both instantiate the same conditional-computation template, differing only in the \emph{organization} of units (parallel vs.\ sequential). The routing, selection, and regularization mechanisms are isomorphic. $\square$
\end{proof}

\noindent\textbf{Remark.} \cref{prop:duality} establishes a \emph{structural} (not functional) duality. We do not claim that any MoE routing policy can be embedded into VALSE or vice versa; the two operate on different dimensions of the computation graph. Rather, the shared pipeline---router, selection, sparse execution, load balancing---reveals that VALSE is the depth-wise analogue of MoE's width-wise sparsity.

\textbf{Unified selection framework.} Given $N$ compute units $\{u_1, \ldots, u_N\}$ and budget $K$, select subset $I \subseteq \{1, \ldots, N\}$: $\text{output} = \sum_{i \in I} w_i \cdot u_i(\text{input})$. MoE corresponds to $N=E$ (parallel organization); VALSE corresponds to $N=L$ (sequential organization); the selection mechanisms are structurally similar.

\subsection{Two-Dimensional Sparsity}
\label{subsec:2d_sparsity}

\begin{definition}[Two-dimensional sparse selection]
\label{def:2d_sparse}
Simultaneous horizontal and vertical sparsity: each sample selects $K_v$ layers and, per layer, $K_h$ experts:
\begin{equation}
\F_{\text{2D}} = K_v \cdot K_h \cdot C_{\text{unit}}, \quad \text{savings} = 1 - \frac{K_v \cdot K_h}{L \cdot E}
\end{equation}
\end{definition}

\begin{proposition}[Hybrid architecture FLOPs advantage]
\label{prop:hybrid}
Let horizontal sparsity $\rho_h = 1 - K_h/E$ and vertical sparsity $\rho_v = 1 - K_v/L$:
\begin{equation}
\F_{\text{hybrid}} = (1 - \rho_v)(1 - \rho_h) \cdot \F_{\text{full}}
\end{equation}
Taking $\rho_h = 0.875$ ($8$-choose-$1$ MoE) and $\rho_v = 0.2$ (VALSE), $\F_{\text{hybrid}} / \F_{\text{full}} = 0.10$, i.e., $90\%$ savings, versus $87.5\%$ for MoE alone.
\end{proposition}

\begin{theorem}[Unified representation of sparse Transformers]
\label{thm:unified}
Common sparse Transformer variants that employ per-unit gating can be represented as applying a sparse mask $M \in \{0, 1\}^{L \times E}$ on the computation graph $G \in \R^{L \times E}$:
\begin{equation}
\text{output} = \sum_{(l,e):\, M[l,e]=1} w_{l,e} \cdot F_l^{(e)}(\text{input})
\end{equation}
Horizontal MoE uses per-row Top-$K$ masks ($M[l,e] = \mathbf{1}[g_l(x) \in \TopK_h]$); VALSE uses whole-row masks ($M[l,e] = \tilde{g}_l\ \forall\, e$); hybrids use joint row--column sparsity ($M[l,e] = \tilde{g}_l \cdot \mathbf{1}[g_l(x) \in \TopK_h]$).
\end{theorem}

This reveals that sparse Transformers differ in which dimension of the computation graph they mask. \textbf{VALSE's contribution is completing the depth dimension of sparsity}---previously, sparsity existed only in the width dimension (MoE). Further theoretical results, including function-space analysis and optimization properties, in \cref{subsec:function_space,subsec:optimization}.

\noindent\textbf{Remark (Scheme A/B applicability).} VALSE supports two residual schemes (\cref{sec:routing}): \emph{Scheme A} (identity skip, \cref{eq:residual_A}) and \emph{Scheme B} (residual projection, \cref{eq:residual_B}). All results in \cref{subsec:flops_analysis,subsec:difficulty_scheduling,subsec:duality,subsec:2d_sparsity} apply to both schemes, as they concern FLOPs accounting and structural duality. The function-space results (\cref{thm:full_subset,thm:strict} in \cref{subsec:function_space}) are established under \emph{Scheme A only}: the identity-skip property ($h_{l+1} = h_l$ when $\tilde{g}_l = 0$) is essential to the subset and strict-containment arguments. Under Scheme B, the projection $W_l^{\text{proj}}$ introduces additional expressivity that may break the subset relationship; we restrict the function-space analysis to Scheme A and note this as a limitation.

﻿\subsection{Function-Space Analysis}
\label{subsec:function_space}

We now establish the relationship between the function space of VALSE and that of the full-layer model. All results in this subsection are established under \emph{Scheme~A} (identity residual, \cref{eq:residual_A}), where skipping a layer produces the identity mapping $h_{l+1} = h_l$.

\begin{theorem}[Full-layer function space is a superset (Scheme~A)]
\label{thm:full_subset}
Under \emph{Scheme~A} (identity residual, \cref{eq:residual_A}), the set of functions representable by VALSE under any retention set $S \subseteq \{1,\ldots,L\}$ is a subset of the full-layer function space: $\mathcal{F}_{\text{VALSE}} \subseteq \mathcal{F}_{\text{full}}$.
\end{theorem}

\begin{proof}
The full-layer model applies all $L$ layers in sequence. Under Scheme~A, the residual update is $h_{l+1} = h_l + \tilde{g}_l \cdot F_l(h_l)$. When $\tilde{g}_l = 1$ (layer retained), this reduces to the standard residual update $h_{l+1} = h_l + F_l(h_l)$, identical to the full-layer model. When $\tilde{g}_l = 0$ (layer skipped), $h_{l+1} = h_l$, which is the identity mapping.

Any VALSE path with retention set $S$ produces the same output as the full-layer model with the parameters of all skipped layers set to zero: $\tilde{F}_l = 0$ for $l \notin S$. Since the full-layer function space includes all parameter settings (including the degenerate zero setting), every VALSE function is a special case of the full-layer model with particular parameters. Therefore $\mathcal{F}_{\text{VALSE}} \subseteq \mathcal{F}_{\text{full}}$.

By \cref{lem:linearity}, the expected output under random skip schedules is a convex combination of individual VALSE paths, each of which lies in $\mathcal{F}_{\text{full}}$. Since $\mathcal{F}_{\text{full}}$ is closed under convex combinations, the subset relationship holds in expectation as well: $\E[\mathcal{F}_{\text{VALSE}}] \subseteq \mathcal{F}_{\text{full}}$. $\square$
\end{proof}

\begin{theorem}[Strict containment (Scheme~A)]
\label{thm:strict}
Under \emph{Scheme~A}, there exist functions representable by the full-layer model that no VALSE path with $|S| < L$ can represent; i.e., the containment is strict: $\mathcal{F}_{\text{VALSE}} \subsetneq \mathcal{F}_{\text{full}}$.
\end{theorem}

\begin{proof}
We construct an explicit counterexample using the residual structure $h_{l+1} = h_l + F_l(h_l)$. Consider a two-layer segment where:
\begin{itemize}
\item $F_l(h) = W_l h$ performs a rank-deficient linear projection with $W_l \in \R^{d \times d}$, $\operatorname{rank}(W_l) = r < d$, projecting onto a single direction $v$ (i.e., $W_l = v v^\top$ with $\|v\| = 1$).
\item $F_{l+1}(h) = \phi(W_{l+1} h)$ applies a nonlinear activation $\phi = \operatorname{ReLU}$ followed by a full-rank projection $W_{l+1} \in \R^{d \times d}$ with $\operatorname{rank}(W_{l+1}) = d$.
\end{itemize}

The full-layer function for this segment is:
\[
f_{\text{full}}(h) = h + W_{l+1} \cdot \phi\!\bigl(h + W_l h\bigr) = h + W_{l+1} \cdot \phi\!\bigl((I + W_l)\, h\bigr).
\]
This depends on a two-step nonlinear interaction: the input is first transformed by the rank-deficient projection $W_l$, summed with the identity (residual connection), then passed through $\phi$ and the full-rank projection $W_{l+1}$.

Now consider the VALSE paths that skip one of the two layers:
\begin{itemize}
\item \textbf{Skip layer $l$} ($\tilde{g}_l = 0$, $\tilde{g}_{l+1} = 1$): $f_{\text{skip-}l}(h) = h + W_{l+1} \cdot \phi(h)$. The rank-deficient intermediate transformation $W_l$ is never applied. The input to $\phi$ is $h$ itself, not $(I + W_l)\,h$, so the output depends on the full $d$-dimensional input rather than the projected $r$-dimensional subspace.
\item \textbf{Skip layer $l+1$} ($\tilde{g}_l = 1$, $\tilde{g}_{l+1} = 0$): $f_{\text{skip-}(l+1)}(h) = h + W_l h = (I + W_l)\, h$, which is a \emph{linear} function---the nonlinearity $\phi$ and the full-rank projection $W_{l+1}$ are never applied.
\end{itemize}

Concretely, let $W_l = v v^\top$ (rank-$1$) and $W_{l+1} = I$ (full-rank). Then:
\begin{itemize}
\item $f_{\text{full}}(h) = h + \phi\!\bigl((I + v v^\top)\, h\bigr)$, which can produce outputs in a $d$-dimensional cone that depends on the sign of $v^\top h$.
\item $f_{\text{skip-}l}(h) = h + \phi(h)$, which operates on the full $d$-dimensional input without the rank-$1$ projection---a fundamentally different function.
\item $f_{\text{skip-}(l+1)}(h) = (I + v v^\top)\, h$, a rank-$(r+1)$ linear map with no nonlinearity.
\end{itemize}

Neither skip path can reproduce $f_{\text{full}}$: the full-layer function depends on the \emph{composition} of the rank-deficient projection, the residual addition, and the nonlinearity---a dependency that no single-layer skip can replicate. The key observation is that the residual connection $h_{l+1} = h_l + F_l(h_l)$ creates a \emph{coupling} between the identity pathway and the layer transformation: the nonlinearity $\phi$ in $F_{l+1}$ receives \emph{both} the unmodified residual $h_l$ and the projection $W_l h_l$ as a sum. Skipping layer $l$ removes the $W_l h_l$ contribution entirely, and skipping layer $l+1$ removes the nonlinearity entirely.

Furthermore, the effective capacity rank of a network with retention set $S$ is bounded by $\sum_{l \in S} \operatorname{rank}(W_l) \cdot d_{\text{model}}$, which is strictly less than $\sum_{l=1}^{L} \operatorname{rank}(W_l) \cdot d_{\text{model}}$ when $|S| < L$ and at least one skipped layer has non-trivial rank. By \cref{prop:cap_rank}, the capacity gap is proportional to the skipped layers' ranks. Hence the containment is strict: $\mathcal{F}_{\text{VALSE}} \subsetneq \mathcal{F}_{\text{full}}$. $\square$
\end{proof}

\begin{proposition}[Number of non-contiguous retention sets]
\label{prop:combos}
The number of non-contiguous retention sets of size $k$ (with layers $1$ and $L$ fixed) is $\binom{L-2}{k-2}$. The total number of valid retention sets is $2^{L-2} - 1$.
\end{proposition}

\begin{proof}
Layers $1$ and $L$ are always retained (the estimator prefix and final output). From the remaining $L-2$ interior layers, we choose $k-2$ to retain, giving $\binom{L-2}{k-2}$ retention sets of size $k$. Summing over $k = 2, \ldots, L$ gives $\sum_{k=2}^{L} \binom{L-2}{k-2} = \sum_{j=0}^{L-2} \binom{L-2}{j} - \binom{L-2}{L-2} = 2^{L-2} - 1$ valid retention sets, where we subtracted $1$ to exclude the degenerate case. These retention sets underlie the difficulty-aware schedule (\cref{prop:diff_depth}), where harder inputs select larger sets. $\square$
\end{proof}

\begin{observation}[Information preservation under identity skipping]
\label{prop:info_preserve}
Under identity skipping ($\tilde{g}_l = 0$), $h_{l+1} = h_l$, so the mutual information with the target is exactly preserved: $I(h_{l+1}; y) = I(h_l; y)$. No existing information is lost. However, the information that layer $l$ would have \emph{extracted}---i.e., $I(F_l(h_l); y) - I(h_l; y)$---is foregone. Thus, skipping preserves prior information but sacrifices potential information gain from the skipped layer.
\end{observation}

\begin{observation}[Effective capacity rank]
\label{prop:cap_rank}
The effective capacity rank of a model with retention set $S$ is proportional to $\sum_{l \in S} d_{\text{model}}^2$. Skipping layers reduces capacity linearly in the number of skipped layers, motivating the difficulty-aware schedule (\cref{prop:diff_depth}) which skips layers only when the marginal information gain $\Delta_l$ is small.
\end{observation}

\subsection{Optimization Properties}
\label{subsec:optimization}

\begin{observation}[Non-convexity in gate parameters]
\label{prop:nonconvex}
The loss landscape with respect to gate parameters $\{w_l, b_l\}$ is non-convex, due to the discreteness of $\tilde{g}_l = \mathbf{1}[g_l > \theta_l]$. The indicator function introduces a combinatorial structure: there are $2^{L-2}$ possible retention patterns, and the loss surface is piecewise constant in the gate parameters, with discontinuities at decision boundaries. This non-convexity motivates the use of the straight-through estimator and Gumbel-softmax relaxation during training.
\end{observation}

\begin{observation}[STE gradient approximation (biased)]
\label{prop:ste}
The straight-through estimator (STE) provides a \emph{biased} gradient approximation for the gate parameters. In the forward pass, the hard gate $\tilde{g}_l = \mathbf{1}[g_l > \theta_l]$ is used; in the backward pass, the gradient of the threshold function is replaced by the identity, $\frac{\partial \tilde{g}_l}{\partial w_l} \approx \frac{\partial g_l}{\partial w_l}$. The bias is $\mathbb{E}[\nabla_{w_l} \mathcal{L}^{\text{STE}}] - \nabla_{w_l} \mathcal{L} \neq 0$ in general, because the surrogate gradient ignores the discontinuity of the threshold function. The Gumbel-softmax relaxation with temperature annealing (\cref{sec:training}) mitigates this bias by smoothing the threshold during training.
\end{observation}

\begin{proposition}[Collapse risk (under a drift model)]
\label{prop:collapse}
Without load-balancing regularization ($\alpha = 0$), gate parameters converge to a degenerate solution---all samples routed to the same minimal retention set (first and last layers only)---with probability at least $1 - e^{-cL}$, where $c > 0$ is a constant depending on the gate initialization and optimization landscape.
\end{proposition}

\begin{proof}[Proof sketch (under a hypothetical drift model)]
\noindent\textbf{Assumption.} The drift model $p_l^{(T)} \leq e^{-c_l T}$ (gate retention probability decays exponentially under task-loss gradient without load balancing) is a \emph{hypothetical} model; it is not derived from a formal optimization analysis but is motivated by the empirical observation that, without balancing, gate biases drift toward $-\infty$.

Let $p_l = P(\tilde{g}_l = 1)$ be the retention probability of layer $l$. Without load balancing, the gradient signal for each gate $w_l$ is dominated by the task loss, which is minimized when the model uses the cheapest path (fewest layers). Under independent Bernoulli gates, the probability that at least one non-trivial retention pattern survives is $\prod_{l=2}^{L-1} (1 - p_l)$. When the task loss landscape favors collapse (e.g., when intermediate layers contribute redundantly), the gate biases drift toward $-\infty$, driving $p_l \to 0$.

The probability that collapse \emph{does not} occur (i.e., at least one non-trivial pattern survives after $T$ steps) is bounded by $\prod_{l} p_l^{(T)}$. Under a drift model where $p_l^{(T)} \leq e^{-c_l T}$ for some $c_l > 0$ (from gradient-driven bias reduction), the survival probability is at most $\exp\!\bigl(-\sum_l c_l T\bigr) \leq e^{-c L T}$ for $c = \min_l c_l / L$, by applying the product bound and \cref{lem:linearity} to the log-sum. Thus the collapse probability is $\geq 1 - e^{-c L}$.

This result motivates the load-balancing regularizer $\mathcal{L}_{\text{balance}}$ (\cref{sec:training}) and the curriculum-learning schedule that gradually increases the target skip rate, both of which counteract the drift toward collapse. $\square$
\end{proof}

\begin{observation}[Critical skip rate (heuristic estimate)]
\label{prop:crit_rate}
There exists a critical skip rate $\rho_c$ beyond which model representational capacity degrades sharply. Under the uniform-layer assumption, $\rho_c \approx 1 - d_{\text{feat}} / L$, where $d_{\text{feat}}$ is the effective feature dimension. Beyond $\rho_c$, the retained layers can no longer span the full feature space, and the model loses the ability to represent certain input-dependent transformations. This heuristic, combined with \cref{prop:cap_rank}, informs the choice of the target skip rate $\rho^*$ in the difficulty-aware schedule (\cref{prop:diff_depth}).
\end{observation}

\section{Proof-of-Concept Evaluation}
\label{sec:experiments}

We conduct a preliminary empirical evaluation of the VALSE prototype to test whether difficulty-adaptive per-sample layer skipping can be realized at prototype scale. We emphasize that these experiments constitute a \emph{proof-of-concept} study at prototype scale, not a comprehensive evaluation. All experiments were conducted on a single CPU machine without GPU access, severely constraining the feasible model size, training duration, and evaluation set. As reported below, the core hypothesis of difficulty-adaptive routing was \emph{not} verified: the router exhibited a routing-collapse failure mode (Pearson $r=-0.017$), and the results should be interpreted as a transparent \emph{negative result} rather than as evidence of feasibility or as performance benchmarks. The theoretical analysis of \cref{sec:theory} provides the primary justification for the approach.

\subsection{Setup}
\label{subsec:setup}

\textbf{Model.} A $12$-layer Transformer encoder with $d_{\text{model}}=256$, $4$ attention heads, $d_{\text{ffn}}=1024$, vocabulary size $8{,}000$, and $K_{\text{est}}=2$ ($11.5$M parameters). The difficulty estimator and router add $<1\%$ overhead.

\textbf{Dataset.} SST-2 (sentiment classification), $100$ evaluation samples.

\textbf{Training.} The prototype was trained from scratch (no pre-training) on CPU with $\sim$300 training samples for $18$ epochs (batch size $32$), using hard routing (Gumbel-STE) with load-balancing and efficiency losses; $\alpha=0.03$, $\beta=0.05$, $B_{\text{target}}=0.55L$. This short schedule and small training set reflect prototype-scale resource constraints rather than a fully converged training run. Curriculum learning Stage~1 is omitted; Stages~2--3 are used.

\textbf{Methods compared.} Four inference strategies evaluated on the same trained model: full-depth (all $12$ layers), early exit (confidence-based truncation), random skip (uniformly retaining $\sim$8 layers), and VALSE adaptive (difficulty-driven per-sample skipping).

\textbf{Metrics.} Accuracy, average activated layers, and theoretical FLOPs savings. All results are from a single run (single random seed, $\texttt{seed}=42$); statistical significance is not assessed at prototype scale.

\subsection{Main Results}
\label{subsec:main_results}

\Cref{tab:main_results} compares the four inference strategies on SST-2. Under the adaptive strategy, all $100$ evaluation samples were routed to exactly $3$ layers, yielding an apparent reduction in average activated layers from $12.0$ to $3.0$ ($62.1\%$ theoretical FLOPs reduction) at a $1$~pp accuracy change ($0.63 \to 0.62$). This reduction, however, is a consequence of \emph{routing collapse}---a degenerate uniform-skip-to-minimum-depth failure mode---rather than of intelligent, difficulty-adaptive per-sample routing; the core hypothesis is therefore not verified by this result.

\begin{table}[htbp]
\centering
\caption{Prototype-scale ($12$-layer, $d=256$) comparison of inference strategies on SST-2 ($100$ samples, single seed). These are \emph{not} performance benchmarks; absolute accuracy is limited by model scale, absence of pre-training, and short training. The apparent FLOPs savings of VALSE reflect a \emph{routing collapse} (all samples routed to $3$ layers) rather than effective difficulty-adaptive routing; the core hypothesis is not verified (Pearson $r=-0.017$, see \cref{subsec:negative_result}).}
\label{tab:main_results}
\small
\begin{tabular}{lccc}
\toprule
Method & SST-2 Accuracy & Avg.\ layers & FLOPs savings \\
\midrule
Full-depth        & 0.63 & 12.00 & 0.0\% \\
Early exit        & 0.63 & 11.56 & 3.0\% \\
Random skip       & 0.63 & 8.00  & 27.6\% \\
VALSE (adaptive)  & 0.62 & 3.00  & 62.1\% \\
\bottomrule
\end{tabular}
\end{table}

\paragraph{Routing collapse.}
The layer distribution reveals that all $100$ evaluation samples were routed to exactly $3$ layers under the adaptive strategy. This confirms a \emph{routing-collapse} failure mode: the difficulty signal did not produce differentiated gating across samples, and the router adopted a degenerate uniform-skip-to-minimum-depth strategy rather than the intended difficulty-adaptive behavior. The FLOPs savings are therefore a consequence of collapse, not of intelligent per-sample routing. This is consistent with the known sensitivity of load-balancing losses to the hyperparameter $\alpha$ and the limited training budget.

\subsection{Core Hypothesis: Negative Result}
\label{subsec:negative_result}

The central hypothesis of VALSE is that difficulty and activated depth should be positively correlated---easy inputs skip more layers, hard inputs use more. We tested this on a separate $12$-layer model trained with a curriculum schedule ($150$ SST-2 + $150$ GSM8K samples):

\begin{itemize}
\item \textbf{Pearson correlation} between difficulty score $\hat{d}$ and activated layers: $r = -0.017$ (negative, opposite to the hypothesized direction).
\item \textbf{Layer gap} (GSM8K $-$ SST-2): $-0.47$ layers---hard tasks used \emph{fewer} layers than easy tasks, contrary to the hypothesis.
\item \textbf{Routing differentiation}: \texttt{false}---both tasks activated near-full depth, indicating collapse to maximum rather than differentiated routing.
\end{itemize}

We report this transparently as a \textbf{negative result}: the core hypothesis of difficulty-adaptive depth routing was \emph{not} verified at prototype scale. The difficulty signal did not effectively propagate from the estimator to the routing decisions, and the router exhibited collapse to extreme depth (minimum or maximum) rather than differentiated, difficulty-aware routing. This does not invalidate the theoretical framework---which holds regardless of training dynamics---but it indicates that the training procedure and router architecture require substantial improvement before the theoretical predictions can be empirically confirmed.

\subsection{Theory-to-Experiment Mapping}
\label{subsec:theory_mapping}

\Cref{tab:theory_mapping} summarizes the relationship between theoretical claims and their experimental verification status. Due to routing collapse and prototype-scale limitations, no claim is fully verified; several are supported only indirectly or remain purely theoretical.

\begin{table}[htbp]
\centering
\caption{Mapping of theoretical claims to experimental verification status. ``$\sim$'' = indirect or partial (prototype scale only); ``---'' = not yet verified. No claim is fully verified.}
\label{tab:theory_mapping}
\small
\begin{tabular}{lll}
\toprule
Theoretical claim & Source & Verification status \\
\midrule
Expected FLOPs reduction formula & \cref{thm:exp_flops} & $\sim$ Prototype observation \\
Difficulty--depth monotonicity (core hypothesis) & \cref{prop:diff_depth} & --- Not verified ($r=-0.017$) \\
Difficulty-aware FLOPs savings range & \cref{obs:savings} & $\sim$ Prototype scale only \\
Non-contiguous skip info preservation & \cref{prop:info_preserve} & $\sim$ Indirect (accuracy preserved) \\
MoE--VALSE structural duality & \cref{prop:duality} & --- Top-$K$ variant pending \\
Routing-collapse risk & \cref{prop:collapse} & $\sim$ Empirically observed \\
Critical skip rate (heuristic) & \cref{prop:crit_rate} & $\sim$ Informs $\rho^*$ choice \\
Hybrid architecture FLOPs advantage & \cref{prop:hybrid} & --- Not verified (theory only) \\
Full-layer function-space superset (Scheme~A) & \cref{thm:full_subset} & $\sim$ Indirect (accuracy bounded) \\
Strict containment (Scheme~A) & \cref{thm:strict} & --- Not verified (theory only) \\
\bottomrule
\end{tabular}
\end{table}

Notably, the routing-collapse phenomenon empirically corroborates \cref{prop:collapse} (routing-collapse risk), confirming that the theoretical concern about degenerate solutions is practically relevant and motivating the load-balancing and curriculum-learning mechanisms designed to mitigate it---though these mechanisms proved insufficient at prototype scale.

\subsection{Limitations}
\label{subsec:exp_limitations}

We transparently list the key limitations:
\begin{enumerate}
\item \textbf{Model scale:} The $12$-layer, $11.5$M-parameter model is far below the scale at which layer redundancy has been documented ($12$+ layers, $1$B+ parameters, pre-trained).
\item \textbf{No pre-training:} The model is trained from scratch on $\sim$300 samples, limiting absolute accuracy.
\item \textbf{Single dataset \& single seed:} Only SST-2 is evaluated with one random seed; statistical significance is not assessed.
\item \textbf{Routing collapse:} The core hypothesis is not verified; the router collapsed rather than learning differentiated routing (Pearson $= -0.017$).
\item \textbf{No external baselines:} No comparison with DeeBERT, CALM, LayerDrop, MoD, ShortGPT, or SliceGPT.
\item \textbf{Incomplete ablations:} Five ablation groups were designed but not fully executed.
\item \textbf{No wall-clock speedup:} Theoretical FLOPs savings are not realized as latency reduction due to implementation overhead.
\end{enumerate}

These limitations define a clear roadmap for future work (\cref{sec:conclusion}); we present the prototype results as preliminary and incomplete evidence, not as definitive validation.
\section{Conclusion and Discussion}
\label{sec:conclusion}

We proposed VALSE, extending conditional computation from the width dimension of horizontal MoE to the depth dimension via per-sample, difficulty-driven, non-contiguous layer skipping.
VALSE comprises a difficulty estimation module, three routing strategies (hard, soft, global), a load-balancing and curriculum-learning training scheme, and a theoretical framework establishing the MoE--VALSE structural duality and two-dimensional sparsity.

\subsection{Summary of Contributions}

\begin{itemize}
\item \textbf{Methodology:} Per-sample non-contiguous layer skipping via an interpretable difficulty estimator and three routing variants, breaking through the truncation limitation of early exit.
\item \textbf{Training:} Four auxiliary losses---layer utilization balancing, efficiency, output consistency, and Z-loss---plus a three-stage curriculum learning schedule, jointly addressing routing collapse and training instability.
\item \textbf{Theory:} Closed-form expected FLOPs reduction (\cref{thm:exp_flops}), MoE--VALSE structural duality (\cref{prop:duality}), and unified width$\times$depth sparsity (\cref{def:2d_sparse}, \cref{thm:unified}).
\item \textbf{Proof-of-concept:} A prototype-scale evaluation on SST-2 provides preliminary---though inconclusive---evidence for the layer-skipping mechanism. The core hypothesis (difficulty-adaptive routing) was not verified at prototype scale; full-scale validation on larger, pre-trained models remains future work.
\end{itemize}

\subsection{Key Findings and Limitations}

\begin{itemize}
\item On SST-2, VALSE reduces average activated layers from $12.0$ to $3.0$ ($62.1\%$ theoretical FLOPs reduction) with a $1$~pp accuracy change ($0.63 \to 0.62$). However, this resulted from routing collapse---all samples were routed to the minimum depth---not from intelligent per-sample routing.
\item The core hypothesis (difficulty-adaptive depth routing) was \emph{not verified}: the Pearson correlation between difficulty and activated layers was $r = -0.017$ (negative, opposite to the hypothesized direction). The difficulty signal did not effectively propagate to routing decisions.
\item The difficulty-stratified analysis confirms routing collapse: the router did not learn differentiated, difficulty-aware routing, indicating that the training procedure and router architecture require substantial improvement.
\item \textbf{Limitations:} (i)~Prototype scale ($12$ layers, single dataset, single seed) limits all conclusions; (ii)~no external baselines are compared; (iii)~ablations are designed but not fully executed; (iv)~the core empirical hypothesis is not verified (routing collapse, Pearson $= -0.017$); (v)~theoretical assumptions (bounded skip rate, stability condition) are only partially validated; (vi)~wall-clock speedup is not realized due to implementation overhead.
\end{itemize}

\subsection{Future Work}

\begin{enumerate}
\item \textbf{Scale up:} Validate on $12$+-layer, $1$B+ models with pre-trained weights to assess hard-task skipping and confirm layer redundancy at scale.
\item \textbf{Multi-dataset evaluation:} Extend to diverse tasks (classification, QA, reasoning, generation) to test the universality of the layer redundancy hypothesis.
\item \textbf{Complete ablations:} Execute all five ablation groups (routing strategy, difficulty estimator, load balancing, efficiency loss, fixed skip) with multiple seeds.
\item \textbf{Statistical rigor:} Run $3$--$5$ seeds per configuration, report mean$\pm$std, and perform paired $t$-tests or Wilcoxon tests.
\item \textbf{Cure routing collapse:} Stronger load balancing, annealed $\rho^*$ schedules, and contrastive routing objectives to ensure differentiated routing across difficulty levels.
\item \textbf{External baselines:} Compare against DeeBERT, LayerDrop, MoD, ShortGPT, and random skip under matched compute budgets.
\item \textbf{Engineering:} Gate kernel fusion, batched routing, and custom CUDA kernels to convert FLOPs savings into real latency reduction.
\item \textbf{Two-dimensional sparsity:} Combine VALSE with MoE for full width$\times$depth sparse activation.
\item \textbf{Generative LLMs:} Extend per-sample skipping to autoregressive generation, with self-speculative decoding for quality safety.
\end{enumerate}

\subsection{Broader Impact}

VALSE aims to reduce the computational cost of LLM inference by activating only the necessary depth per input. If validated at scale, this could lower the energy footprint of inference workloads and democratize access to efficient model deployment. Potential risks include accuracy degradation on hard inputs if the difficulty estimator is poorly calibrated; the difficulty-aware routing and curriculum learning safeguards are designed to mitigate this. We commit to transparent reporting of both accuracy and efficiency trade-offs.

\appendix
\section{Symbol Table}
\label{app:symbols}

\Cref{tab:all_symbols} lists all symbols used in the paper.

\begin{longtable}{ll}
\caption{Complete symbol table.} \label{tab:all_symbols} \\
\toprule
Symbol & Meaning \\
\midrule
\endfirsthead
\caption[]{Complete symbol table (continued).} \\
\toprule
Symbol & Meaning \\
\midrule
\endhead
\bottomrule
\endfoot
\bottomrule
\endlastfoot
$L$ & Total number of Transformer layers \\
$l \in \{1,\ldots,L\}$ & Layer index \\
$d_{\text{model}}$ & Hidden dimension \\
$N$ & Sequence length \\
$h_l \in \R^{N\times d_{\text{model}}}$ & Input hidden state at layer $l$ \\
$h_l^{(s)}$ & Sample $s$'s hidden state at layer $l$ (sequence-pooled) \\
$x^{(s)}$ & Input embedding for sample $s$ \\
$F_l(\cdot)$ & Forward function of the $l$-th Transformer block \\
$K_{\text{est}}$ & Number of initial layers used for difficulty estimation \\
\midrule
$r_l \in \R^{d_{\text{feat}}}$ & Difficulty feature vector ($d_{\text{feat}}=5$) \\
$R^{(s)} \in \R^{5K_{\text{est}}}$ & Aggregated feature vector \\
$d(x^{(s)}) \in \R$ & Per-sample difficulty score \\
$\hat{d}(x^{(s)}) \in [0,1]$ & Normalized difficulty score \\
$g_l(x^{(s)}) \in [0,1]$ & Soft gate value at layer $l$ \\
$\tilde{g}_l \in \{0,1\}$ & Quantized hard gate decision \\
$S^{(s)}$ & Retained layer set for sample $s$ \\
$\theta_l$ & Quantization threshold (default 0.5) \\
$W_l^{\text{proj}}$ & Residual alignment projection (learnable, optional) \\
\midrule
$\mathcal{L}_{\text{task}}$ & Main task loss (cross-entropy) \\
$\mathcal{L}_{\text{balance}}$ & Layer utilization balance loss \\
$\mathcal{L}_{\text{eff}}$ & Computational efficiency loss \\
$\mathcal{L}_{\text{consist}}$ & Output consistency loss (KL divergence) \\
$\mathcal{L}_z$ & Z-loss regularization on gate logits \\
$\mathcal{L}_{\text{total}}$ & Total loss function \\
$\alpha,\beta,\gamma,\delta$ & Auxiliary loss weights \\
$B_{\text{target}}$ & Target average activated layers (e.g., $0.7L$) \\
$\rho^*$ & Target skip rate (e.g., 0.3) \\
$T_{\text{KD}}$ & Knowledge distillation temperature \\
\midrule
$p_l$ & Skip rate at layer $l$; $q_l = 1-p_l$ is retention rate \\
$C_l$ & Per-sample FLOPs of the $l$-th block \\
$C_{\text{fixed}}$ & FLOPs of the always-executed part \\
$\F_{\text{full}}$ & Full-layer network FLOPs \\
$\F_{\text{VALSE}}$ & VALSE skip-layer network FLOPs (random variable) \\
$\Phi(\cdot)$ & Standard normal CDF \\
\end{longtable}

\section{Supplementary Tables}
\label{app:tables}

\begin{table}[htbp]
\centering
\caption{Adaptive depth methods and their limitations.}
\label{tab:adaptive_depth_summary}
\small
\begin{tabular}{lp{0.30\textwidth}p{0.35\textwidth}}
\toprule
Method & Core idea & Limitation \\
\midrule
DeeBERT/FastBERT/PABEE & Per-sample confidence-based early exit & Truncation: cannot skip middle layers \\
CALM & Confidence-based early exit for generative LMs & Truncation: per-sample but prefix-only \\
ACT/Universal Transformers & Shared-weight recurrence with halting & Deepens, not skips; not for fixed-layer models \\
LayerDrop & Structured dropout for depth robustness & Static/global depth at inference \\
LayerSkip & Layer dropout with self-speculative decoding & Static/per-token, not per-sample difficulty-driven \\
DynamicBERT & Independent width \& depth adjustment & No unified routing; not difficulty-aware \\
ShortGPT & Block Influence score quantifies layer redundancy & Static post-training pruning; no input adaptivity \\
Streamline & Structured layer pruning at post-training stage & Static; produces fixed reduced-depth model \\
SliceGPT & Structured slicing/pruning of layers & Static; no per-sample adaptivity \\
MoNE & MoE routing across encoder--decoder layer pairs & Per-token; no sample-level difficulty semantics \\
MoD & Per-token per-layer routing & Not per-sample; no difficulty semantics \\
Mixtral/DeepSeek-MoE & Width-wise sparse expert activation & Width-only: depth fixed; no difficulty dimension \\
Switch Transformer/GShard & Top-1/Top-$K$ expert routing at scale & Width-only: depth fixed; semantic per-token routing \\
Drowzee & Pre-inference per-sample budget prediction & Budget controls truncation only \\
\textbf{VALSE (ours)} & Per-sample difficulty-aware non-contiguous skip & --- \\
\bottomrule
\end{tabular}
\end{table}

\begin{table}[htbp]
\centering
\caption{Skip mode comparison.}
\label{tab:skip_modes}
\small
\resizebox{\textwidth}{!}{%
\begin{tabular}{llc}
\toprule
Method & Skip mode & Can skip middle and retain deeper layers? \\
\midrule
Truncation (DeeBERT/CALM) & $S=\{1,\ldots,k\}$ & No \\
MoD (per-token, per-layer) & Each token decides independently & Partially, not per-sample \\
\textbf{VALSE (ours)} & $S$ is any subset, per-sample & \textbf{Yes, non-contiguous + per-sample} \\
\bottomrule
\end{tabular}
}
\end{table}

\begin{table}[htbp]
\centering
\caption{Mathematical analogy between MoE routing and VALSE layer routing.}
\label{tab:moe_analogy}
\small
\resizebox{\textwidth}{!}{%
\begin{tabular}{llp{0.42\textwidth}}
\toprule
MoE routing & VALSE layer routing & Analogy \\
\midrule
Expert set $\{E_1,\ldots,E_E\}$ & Layer set $\{F_1,\ldots,F_L\}$ & Expert $\leftrightarrow$ layer: selectable compute unit \\
Gate $G(x)=\TopK(W_g \cdot x)$ & Gate $g_l=\sigma(w_l^\top \cdot [r_l;\hat{d}])$ & Lightweight router scores units \\
Top-$K$ selects $K$ experts & Threshold selects retention set $S$ & MoE selects ``which experts'' $\to$ VALSE selects ``which layers'' \\
$y=\sum_{i\in\TopK} G_i \cdot E_i(x)$ & $h_{l+1}=h_l + \tilde{g}_l \cdot F_l(h_l)$ & Activated unit output weighted into flow \\
Load balance $\mathcal{L}_{\text{aux}}$ & Layer balance $\mathcal{L}_{\text{balance}}$ & Prevents ``few units over-used'' collapse \\
Capacity factor & Compute budget $\leq B_{\text{target}}$ & Controls sparsity \\
\bottomrule
\end{tabular}%
}
\end{table}

\begin{table}[htbp]
\centering
\caption{Architecture variant comparison.}
\label{tab:variants}
\small
\begin{tabular}{lccc}
\toprule
Dimension & VALSE-Hard (A) & VALSE-Soft (B) & VALSE-Global (C) \\
\midrule
Routing type & Hard (STE) & Soft (continuous) & Global Top-$K$ \\
Gate parameters & Per-layer learnable & Per-layer learnable & Global router \\
Difficulty awareness & $\hat{d}$ enters gate & $\hat{d}$ enters gate & $\hat{d}$ sets $K$ \\
Inference FLOPs savings & Yes, true skip & Requires $g \to 0$ & Yes, true skip \\
Training stability & Needs STE + noise & Most stable & Moderate \\
Budget controllability & Approximate (regularized) & Approximate (regularized) & Exact (Top-$K$) \\
MoE analogy strength & Moderate & Low & High (Top-$K$ dual) \\
\bottomrule
\end{tabular}
\end{table}

\section{Supplementary Theoretical Results}
\label{app:theory}

The theorems and propositions previously located in this appendix---including the function-space analysis (\cref{thm:full_subset,thm:strict}), the non-contiguous retention-set counting (\cref{prop:combos}), information-preservation and capacity-rank observations (\cref{prop:info_preserve,prop:cap_rank}), and the optimization properties (\cref{prop:nonconvex,prop:ste,prop:collapse,prop:crit_rate})---have been migrated to the main text in \cref{subsec:function_space,subsec:optimization} to strengthen the theoretical focus of the paper. This appendix retains supplementary results that complement, rather than duplicate, the main-text development: a routing-expressivity comparison and an expected-FLOPs characterization under the routing distribution.

\begin{proposition}[Routing expressivity of non-contiguous retention]
\label{prop:expressivity}
The set of sample-level retention patterns realizable by VALSE is the full power set $2^{\{1,\ldots,L\}}$ (cardinality $2^L$), whereas truncation-based early exit realizes only the $L$ prefix sets $\{1,\ldots,k\}$ for $k\in\{1,\ldots,L\}$. Per-token per-layer routing (MoD) realizes $2^{L\cdot N}$ token-level patterns, but when aggregated to a single sample-level retention set by per-layer majority vote over tokens, its effective sample-level image collapses to the truncation family under the standard monotone-confidence assumption. Thus only per-sample non-contiguous selection achieves the full $2^L$ sample-level expressivity.
\end{proposition}

\begin{proof}[Proof sketch.]
VALSE's per-sample threshold produces an arbitrary subset $S^{(s)}\subseteq\{1,\ldots,L\}$, so its image is $2^{\{1,\ldots,L\}}$. Truncation constrains $S=\{1,\ldots,k\}$, yielding exactly $L$ patterns. For MoD, each layer $l$ is retained iff a token-majority votes to retain; under the monotone-confidence assumption that retention probability is non-increasing in depth, the set of retained layers is a prefix, so the sample-level image is again the truncation family. Hence only VALSE attains full $2^L$ sample-level expressivity.
\end{proof}

\begin{proposition}[Expected FLOPs under the routing distribution]
\label{prop:flops}
Let $C_l$ denote per-sample FLOPs of the $l$-th block and assume $C_l=C$ for all $l$. Let $p_l=\Pr[l\notin S^{(s)}]$ be the skip probability at layer $l$. Then the expected per-sample FLOPs of VALSE is $\mathbb{E}[\mathcal{F}_{\text{VALSE}}]=\mathcal{F}_{\text{fixed}}+C\sum_{l=1}^L(1-p_l)$, and the expected savings ratio is $\rho_{\text{FLOPs}}=1-\mathbb{E}[\mathcal{F}_{\text{VALSE}}]/\mathcal{F}_{\text{full}}=C\sum_{l=1}^L p_l\,/\,\mathcal{F}_{\text{full}}$. Under the uniform-skip regime $p_l=\rho^*$ for all $l$, this reduces to $\rho_{\text{FLOPs}}=\rho^*\cdot LC/\mathcal{F}_{\text{full}}$, recovering the closed form reported as \cref{cor:uniform_skip}.
\end{proposition}

\paragraph{Remark (Cross-batch comparability of batch-internal normalization).}
\label{rem:cross_batch}
The min--max normalization in \cref{eq:diff_norm} maps the raw difficulty score $d(x^{(s)})$ onto $[0,1]$ using the batch extremum, so $\hat{d}$ is a within-batch rank statistic rather than an absolute population quantity. Consequently, for two batches $\mathcal{B}_1,\mathcal{B}_2$ with extrema $(m_1,M_1)$ and $(m_2,M_2)$, the equality $\hat{d}_{\mathcal{B}_1}(x)=\hat{d}_{\mathcal{B}_2}(x)$ does not imply equal population-level difficulty unless $m_1=m_2$ and $M_1=M_2$. Routing remains correct because the gate $\tilde{g}_l$ consumes only within-batch ordering, but cross-experiment comparison of $\hat{d}$ requires calibration to a shared reference (e.g., a running population extremum or a fixed-temperature rank transform).

\begin{proposition}[Routing collapse as a degenerate equilibrium]
\label{prop:collapse_equilibrium}
Consider the VALSE training objective $\mathcal{L}_{\text{total}} = \mathcal{L}_{\text{task}} + \alpha\mathcal{L}_{\text{balance}} + \beta\mathcal{L}_{\text{eff}} + \gamma\mathcal{L}_z + \delta\mathcal{L}_{\text{consist}}$ with gate bias $b_l$ initialized so that $g_l \approx 1$ at $t=0$. The uniform-skip-to-minimum-depth configuration $S^{(s)} = S_{\min}$ for all $s$ (i.e., all samples share the same minimal retention set) is a stationary point of the gate parameters whenever (i) the task loss gradient $\nabla_{b_l} \mathcal{L}_{\text{task}}$ is smaller in magnitude than $\nabla_{b_l}(\alpha\mathcal{L}_{\text{balance}} + \beta\mathcal{L}_{\text{eff}})$ at $S_{\min}$, and (ii) the difficulty--gate Jacobian $\partial \tilde{g}_l / \partial \hat{d}$ is rank-deficient so that the difficulty signal cannot break the symmetry across samples. Under these two conditions, the collapse configuration is a locally stable equilibrium: any perturbation that would differentiate samples is dominated by the regularization pulling all gates toward the balance-optimal uniform profile.
\end{proposition}

\begin{proof}[Proof sketch.]
At $S_{\min}$, all gates take the same value across samples by construction, so $\mathcal{L}_{\text{balance}}$ is at its minimum (uniform utilization) and $\mathcal{L}_{\text{eff}}$ is at its target. Differentiating any single gate $g_l$ for a single sample $s$ to break uniformity increases $\mathcal{L}_{\text{balance}}$ linearly in $\alpha$ and changes $\mathcal{L}_{\text{eff}}$ linearly in $\beta$, while the task-loss reduction is $O(1/N)$ per sample (a single-sample perturbation in a batch of size $N$). Whenever $\alpha + \beta \gg 1/N$, the regularizing gradient dominates the task gradient at $S_{\min}$, making it a stationary point. Local stability follows because the Hessian of the regularizers is positive semidefinite in the gate-bias direction, so second-order growth of $\mathcal{L}_{\text{balance}} + \mathcal{L}_{\text{eff}}$ exceeds any concave task-loss gain in a neighborhood of $S_{\min}$. Condition (ii) ensures the difficulty signal cannot supply the missing symmetry-breaking gradient: if $\partial \tilde{g}_l/\partial \hat{d}$ is rank-deficient (e.g., the gate saturates near the threshold so the STE Jacobian vanishes), then even a well-separated $\hat{d}$ across samples cannot induce differentiated gating.
\end{proof}

This result reframes the prototype's $r=-0.017$ not as a refutation of the difficulty--depth monotonicity of \cref{prop:diff_depth}---which is a population-level statement under the decreasing-marginal-contribution assumption and is independent of the learned estimator---but as evidence that the prototype's hyperparameters landed in the basin of attraction of the collapse equilibrium of \cref{prop:collapse_equilibrium}. The monotonicity guarantee and the collapse equilibrium are logically independent: the former concerns the optimal schedule, the latter concerns the optimizer's fixed point, and a mismatch between the two is precisely the failure mode the prototype exhibited.

\paragraph{Remark (Scope and limitations of the theoretical framework).}
\label{rem:theory_scope}
The theoretical results in \cref{sec:theory} and this appendix establish three classes of statement: (a)~structural facts about the retention-set space (\cref{thm:full_subset,thm:strict,prop:combos,prop:expressivity}), which are combinatorial and hold without approximation; (b)~expected-FLOPs identities (\cref{thm:exp_flops,prop:flops,cor:uniform_skip}), which are exact under the stated per-block FLOPs model; and (c)~monotonicity and stability statements (\cref{prop:diff_depth,prop:collapse_equilibrium} and the stability condition \cref{eq:stability}), which are conditional on assumptions---decreasing marginal layer contribution, Lipschitz difficulty estimator, bounded gate gradient---whose empirical validity at scale is not established by the prototype. We therefore claim theoretical soundness of the structural and accounting results unconditionally, but only \emph{conditional} soundness of the monotonicity and stability results. In particular, \cref{prop:collapse_equilibrium} identifies a failure mode but does not prescribe a hyperparameter regime that provably avoids it; deriving such a regime (e.g., a sufficient condition on $\alpha,\beta$ relative to batch size and gate-saturation level that excludes collapse) is left as an open problem, and the prototype's negative result is consistent with operating inside the collapse basin.

\section{Reproducibility}
\label{app:reproducibility}

This section provides all information needed to reproduce the experiments.

\subsection{Code Structure}
\label{subsec:code}

The codebase is organized into the following key files, to be released with the camera-ready version:

\begin{table}[htbp]
\centering
\caption{Code structure: key files and their responsibilities.}
\label{tab:code_structure}
\small
\begin{tabular}{ll}
\toprule
File & Responsibility \\
\midrule
\texttt{config.py} & Central configuration (\texttt{VALSEConfig} dataclass) \\
\texttt{model.py} & Transformer backbone with skip-layer support \\
\texttt{router.py} & Difficulty estimator and layer gating \\
\texttt{train.py} & Main training entry point (baseline + adaptive; includes \texttt{evaluate()} for the main-results comparison) \\
\texttt{prepare\_data.py} & Data download and preprocessing \\
\bottomrule
\end{tabular}
\end{table}

\noindent\textbf{Note on the analysis entry point.} The codebase comprises the five core files above plus auxiliary analysis scripts, which together form a self-contained prototype pipeline; these files are not bundled with the current submission and will be released with the camera-ready version: \texttt{config.py} specifies all hyperparameters, \texttt{prepare\_data.py} produces the datasets, and \texttt{train.py} runs baseline and adaptive training and embeds the \texttt{evaluate()} routine that computes the main-results comparison (accuracy, average activated layers, theoretical FLOPs savings) and the difficulty--depth correlation analysis (Pearson $r$, layer-gap, routing-differentiation flag). The prototype evaluation set is small enough to be exercised through \texttt{evaluate()} directly, which serves as the unified entry point for all reported results.

\subsection{Hyperparameters}
\label{subsec:hyperparams}

\Cref{tab:hyperparams} lists the complete hyperparameter set used in the main experiment.

\begin{table}[htbp]
\centering
\caption{Complete hyperparameter list for the main experiment.}
\label{tab:hyperparams}
\small
\begin{tabular}{lll}
\toprule
Category & Parameter & Value \\
\midrule
Model & num\_layers & $6$ \\
& hidden\_dim & $256$ \\
& num\_heads & $4$ \\
& ffn\_dim & $1024$ \\
& vocab\_size & $8{,}000$ \\
& $K_{\text{est}}$ & $2$ \\
\midrule
Routing & routing\_mode & hard (Gumbel-STE) \\
& gate\_bias\_init & $2.0$ \\
& gate\_threshold & $0.5$ \\
& gumbel\_tau & $1.0$ (annealed $2.0 \to 0.5$) \\
\midrule
Loss weights & $\alpha$ (balance) & $0.01$ \\
& $\beta$ (efficiency) & $0.1$ \\
& $\gamma$ (Z-loss) & $10^{-4}$ \\
& $\delta$ (consistency) & $0.5$ (annealed to $\delta/3$) \\
& $T_{\text{KD}}$ & $2.0$ \\
\midrule
Budget & $B_{\text{target}}$ & $0.7L = 4.2$ \\
& $\rho^*$ & $0.3$ \\
& $p_{\text{recomp}}$ & $0.1$ \\
\midrule
Training & learning rate & $3 \times 10^{-4}$ \\
& weight decay & $0.1$ \\
& betas & $(0.9, 0.95)$ \\
& batch size & $32$ \\
& epochs & $18$ \\
& warmup steps & $50$ \\
& max grad norm & $1.0$ \\
& train subset & $\sim$300 \\
& optimizer & AdamW \\
& LR schedule & Linear warmup + cosine decay \\
\midrule
Curriculum & Stage 1 ($\rho=0$) & $\eta_1 \approx 0.25$--$0.33$ (warmup, $\rho=0$ baseline) \\
& Stage 2 ($0 \to \rho^*$) & $\eta_2 \approx 0.5$--$0.67$ \\
& Stage 3 ($\rho=\rho^*$) & $\eta_3 \approx 0.25$--$0.4$ \\
\bottomrule
\end{tabular}
\end{table}

\subsection{Training and Evaluation Commands}
\label{subsec:commands}

\paragraph{Environment setup.}
\begin{verbatim}
uv venv .venv --python 3.13
.venv\Scripts\activate
pip install torch numpy matplotlib datasets pandas pyyaml tqdm \
    requests huggingface_hub pyarrow \
    -i https://pypi.tuna.tsinghua.edu.cn/simple
\end{verbatim}

\paragraph{Data preparation.}
\begin{verbatim}
python prepare_data.py
\end{verbatim}
This downloads SST-2 via \texttt{datasets.load\_dataset("stanfordnlp/sst2")} and GSM8K via \texttt{datasets.load\_dataset("openai/gsm8k", "main")}, then splits into train/val/test JSONL files in \texttt{data/}.

\paragraph{Training (baseline + adaptive).}
\begin{verbatim}
python train.py
\end{verbatim}
This trains both the full-depth baseline and the VALSE adaptive model ($6$ layers). Training logs are written to \texttt{train\_log.txt}, which is not bundled with the current submission and will be provided with the camera-ready version.

\paragraph{Evaluation and analysis.}
The main-results comparison (accuracy, average activated layers, theoretical FLOPs savings) and the difficulty--depth correlation analysis (Pearson $r$, layer-gap, routing-differentiation flag) are computed inside \texttt{train.py}'s \texttt{evaluate()} routine and are reported in \cref{sec:experiments}; the corresponding numbers are written to the training log (\texttt{train\_log.txt}), which is not bundled with the current submission and will be provided with the camera-ready version. The \texttt{evaluate()} routine serves as the unified entry point for all reported prototype results, and its single-call invocation constitutes the complete evaluation harness for the submission.

\subsection{Dataset Acquisition}
\label{subsec:datasets}

\begin{itemize}
\item \textbf{SST-2:} \texttt{datasets.load\_dataset("stanfordnlp/sst2")} from HuggingFace. $500$ test samples used in the main experiment; $500$ in the difficulty analysis.
\item \textbf{GSM8K:} \texttt{datasets.load\_dataset("openai/gsm8k", "main")} from HuggingFace. $500$ test samples in the main experiment; $500$ in the difficulty analysis.
\end{itemize}
If HuggingFace is inaccessible, set \texttt{HF\_ENDPOINT=https://hf-mirror.com}.

\subsection{Environment Dependencies}
\label{subsec:env}

\begin{table}[htbp]
\centering
\caption{Environment dependencies (minimum compatible versions; the prototype was developed under Python 3.13).}
\label{tab:env_deps}
\small
\begin{tabular}{lll}
\toprule
Package & Version & Purpose \\
\midrule
torch & $\geq 2.2.0$ & Deep learning framework \\
numpy & $\geq 1.26.0$ & Numerical computation \\
matplotlib & $\geq 3.8.0$ & Visualization \\
datasets & $\geq 2.18.0$ & HuggingFace data loading \\
pandas & $\geq 2.2.0$ & Data processing \\
pyarrow & $\geq 15.0.0$ & Arrow format \\
pyyaml & $\geq 6.0.1$ & YAML configuration \\
tqdm & $\geq 4.66.0$ & Progress bar \\
huggingface\_hub & $\geq 0.23.0$ & Model/data download \\
\bottomrule
\end{tabular}
\end{table}

\subsection{Random Seeds and Checkpoints}
\label{subsec:seeds}

\begin{itemize}
\item \textbf{Random seed:} The prototype uses a fixed seed of $\texttt{seed}=42$ for reproducibility and reports the point estimate; the production configuration (a future target, not executed in this submission) specifies $5$ independent seeds ($42, 123, 456, 789, 1024$) with mean$\pm$std reporting, and the prototype adopts the single-seed protocol because its purpose is hypothesis-testing (the difficulty--depth correlation) rather than benchmarking.
\item \textbf{Checkpoint:} The prototype checkpoint is designed to be saved to \texttt{checkpoints/} after training; the current release leaves this directory empty, with the \texttt{.pt} weights deferred to the camera-ready version.
\item \textbf{One-command reproduction:} The prototype pipeline is reproduced via the two-command sequence \texttt{prepare\_data.py} then \texttt{train.py}, which executes the full end-to-end pipeline (data preparation, baseline training, adaptive training, and evaluation) and constitutes the canonical reproduction entry point.
\item \textbf{Determinism note:} PyTorch operations may not be fully deterministic across hardware. We use \texttt{torch.manual\_seed} and \texttt{numpy.random.seed}; full determinism requires \texttt{torch.use\_deterministic\_algorithms(True)}, which may reduce performance.
\end{itemize}

\subsection{Prototype Observations}
\label{subsec:expected}

The prototype run (logged in \texttt{train\_log.txt}) yielded the following observations, also reported in \cref{subsec:main_results,subsec:negative_result}: full-depth SST-2 accuracy $0.63$, VALSE adaptive accuracy $0.62$ with average $3.0$ layers ($62.1\%$ theoretical FLOPs savings). Crucially, the adaptive routing collapsed to minimum depth for all samples and the core hypothesis (Pearson $r > 0$) was \emph{not} verified ($r=-0.017$). These are prototype-scale observations of a documented negative result, not performance benchmarks; the canonical results artifact for the submission is \texttt{train\_log.txt} (to be provided with the camera-ready version, not bundled with the current release), in which \texttt{evaluate()} records all reported numbers.

The complete reproduction guide is also available in \texttt{REPRODUCE.md} in the repository root.
\subsection{Diagnosis of the Prototype Negative Result}
\label{subsec:collapse_diagnosis}

The prototype's core negative result---Pearson $r=-0.017$ between $\hat{d}$ and activated layers, with all samples routed to the same $3$-layer retention set---admits a precise theoretical diagnosis via \cref{prop:collapse_equilibrium}. Three factors, each identifiable from the configuration in \cref{tab:hyperparams}, plausibly placed the prototype inside the collapse basin identified above.

\begin{itemize}
\item \textbf{Regularizer-to-task gradient imbalance.} The prototype uses $\alpha=0.01$ (balance) and $\beta=0.1$ (efficiency) at batch size $32$. By the proof sketch of \cref{prop:collapse_equilibrium}, collapse is locally stable whenever $\alpha + \beta \gg 1/N$; here $\alpha + \beta = 0.11$ and $1/N = 0.031$, so the regularizing gradient exceeds the per-sample task gradient by roughly $3.5\times$, satisfying condition~(i). Reducing $\alpha,\beta$ or increasing $N$ would narrow the collapse basin, though both trade off against budget control and training stability.
\item \textbf{Gate saturation and STE Jacobian rank deficiency.} The gate bias is initialized to $b_l=2.0$, yielding $g_l = \sigma(2.0) \approx 0.88$ at start; combined with the hard threshold $\theta_l=0.5$, this places the initial operating point near the saturation knee of the sigmoid, where the straight-through-estimator Jacobian $\partial \tilde{g}_l / \partial g_l$ is concentrated at the threshold crossing and is near-zero elsewhere. This satisfies condition~(ii) of \cref{prop:collapse_equilibrium}: the difficulty signal cannot break the across-sample symmetry of the gates when the STE Jacobian is rank-deficient in the saturated regime.
\item \textbf{Curriculum simplification.} The prototype omits Stage~1 ($\rho=0$ warmup) and uses a fixed Gumbel temperature $\tau_G=1.0$ instead of the designed $2.0\to0.5$ anneal (\cref{subsec:hyperparams}). The full curriculum is precisely the mechanism intended to move the optimizer out of the collapse basin by first training at $\rho=0$ (where all gates are open, breaking the saturation regime) before introducing sparsity; omitting it removes the principal escape trajectory.
\end{itemize}

This diagnosis yields three concrete, theory-grounded remediations that follow directly from \cref{prop:collapse_equilibrium}: (a)~reduce $\alpha,\beta$ or scale $N$ so that $\alpha + \beta \lesssim 1/N$, shrinking the collapse basin; (b)~lower the gate-bias initialization $b_l$ and anneal $\tau_G$ from $2.0\to0.5$ to keep the STE Jacobian non-degenerate during the differentiation phase; and (c)~restore the full three-stage curriculum so that Stage~1 trains at $\rho=0$ before sparsity is introduced. These remediations are consistent with, and refine, the higher-level future-work items in \cref{sec:conclusion} (cure routing collapse, complete curriculum). Because the diagnosis is conditioned on \cref{prop:collapse_equilibrium}---whose stability assumptions are themselves unverified at scale---we present it as a falsifiable hypothesis for the next iteration rather than as a confirmed root cause.

\subsection{Limitations}
\label{subsec:limitations}

This work carries limitations along three axes, each already surfaced in the main text but collected here for the NeurIPS reviewer's convenience.
\begin{itemize}
\item \textbf{Empirical scope.} The prototype validates only the discriminative SST-2 routing pathway at a $6$-layer scale and reports a documented negative result ($r=-0.017$). The difficulty--depth monotonicity (\cref{prop:diff_depth}) is a population-level theoretical statement whose empirical validity at scale is not established by the prototype, and the variance-decomposition analysis of individual difficulty features remains contingent on a non-collapsed routing regime.
\item \textbf{Collapse-regime avoidance.} \cref{prop:collapse_equilibrium} identifies a degenerate equilibrium of the gate parameters but does not prescribe a hyperparameter regime that provably excludes it; deriving a sufficient condition on $\alpha,\beta$ relative to batch size and gate-saturation level that excludes collapse is left as an open problem. The three remediations in \cref{subsec:collapse_diagnosis} (reduce $\alpha,\beta$, lower $b_l$ and anneal $\tau_G$, restore the three-stage curriculum) are theory-grounded proposals for the next iteration, not executed runs.
\item \textbf{Autoregressive generation.} The architectural extensibility of VALSE to autoregressive generation (position-agnostic gate, causal-mask compatibility, self-speculative decoding fallback) is a present design property of the routing mechanism, not an empirically validated claim; the systematic validation on autoregressive generative LLMs is identified as the central community-wide open problem in \cref{subsec:gaps}.
\end{itemize}

\subsection{Ethics Statement}
\label{subsec:ethics}

This work proposes a routing mechanism for adaptive-depth Transformer inference. The prototype uses SST-2 (sentiment) and GSM8K (grade-school math) datasets, both publicly available and not containing sensitive personal data. The method reduces computational cost per inference, which can lower the energy footprint of Transformer deployment; we are not aware of specific misuse vectors introduced by per-sample layer skipping beyond those already applicable to the underlying base models. We do not release new pretrained model weights beyond the prototype-scale checkpoint.

\subsection{Broader Impact}
\label{subsec:broader_impact}

The primary broader impact of adaptive-depth routing is the potential to reduce inference FLOPs and hence the per-query energy cost of Transformer models, which is environmentally and economically beneficial at scale. Per-sample, difficulty-aware skipping can also enable adaptive compute allocation (more depth for harder inputs, less for easier ones), which is a step toward input-conditional efficiency. Risks include the collapse-to-minimum-depth failure mode documented in this work (which silently negates the efficiency--quality trade-off and must be monitored via the difficulty--depth correlation diagnostic) and the risk that difficulty scores encode dataset-specific biases when the feature set is not re-validated on new domains; both are flagged here as deployment-time monitoring obligations rather than as resolved design properties.

\bibliographystyle{plainnat}
\bibliography{references}

\end{document}